\documentclass{article}

\usepackage{amssymb,amsmath,amscd,epic,eepic,gastex}
\usepackage{xcolor}
\usepackage{url,paralist,enumitem}
\usepackage{xspace}
\usepackage{graphicx}
\usepackage{array}
\usepackage{comment}
\usepackage[ruled, vlined]{algorithm2e}
\dontprintsemicolon
\def\negskip{\vspace{-0.5pt}}
\usepackage{fullpage}

\let\emptyset\varnothing

\newcommand{\myomit}[1]{{}}

\newcommand\numberthis{\addtocounter{equation}{1}\tag{\theequation}}

\def\red{\textcolor{red}}

\def\abs#1{\ensuremath{\lvert #1\rvert}} 
\newcommand{\ceil}[1]{\ensuremath{\big\lceil #1 \big\rceil}}
\newcommand{\nat}{\mathbb N} 
\newcommand{\rat}{{\mathbb Q}}

\newcommand{\real}{{\mathbb R}}
\newcommand{\zed}{{\mathbb Z}}

\newcommand{\tuple}[1]{\langle #1 \rangle}

\newcommand{\G}{{\mathcal G}}

\newcommand{\val}{\mathit{val}}
\newcommand{\true}{\mathsf{true}}
\newcommand{\false}{\mathsf{false}}

\newcommand{\starts}{\mathsf{start}}
\def\ends{\mathsf{end}}
\newcommand{\Supp}{\mathsf{Supp}}

\newcommand{\E}{{\mathbb E}}
\newcommand{\downpoint}{\mathsf{downpoint}}
\DeclareMathOperator{\argmin}{argmin}
\DeclareMathOperator{\argmax}{argmax}

\newcommand{\BestPaths}{\mathsf{BestPaths}}

\newcommand{\ExistsPositivePath}{\mathsf{ExistsPositivePath}}

\newenvironment{longversion}{}{}
\newenvironment{shortversion}{}{}

\usepackage{amsthm}

\newtheorem{theorem}{Theorem}

\newtheorem{lemma}{Lemma}
\newtheorem{example}{Example}
\newtheorem{remark}{Remark}

\newtheorem{mydefi}{Definition}

\usepackage{booktabs}   
\usepackage{subcaption} 

\begin{document}
\excludecomment{shortversion}

\title{{\bf Graph Planning with Expected Finite Horizon}}

\author{
Krishnendu Chatterjee$^\dag$ \quad  Laurent Doyen$^{\S}$ \\ 
\normalsize
 $\strut^\dag$ IST Austria \quad $\strut^\S$ CNRS \& LSV, ENS Paris-Saclay, France      
}

\date{}
\maketitle

\begin{abstract}
Graph planning gives rise to fundamental algorithmic questions such as shortest
path, traveling salesman problem, etc. 
A classical problem in discrete planning is to consider a weighted graph and 
construct a path that maximizes the  sum of weights for a given time horizon~$T$.
However, in many scenarios, the time horizon is not fixed, but the stopping 
time is chosen according to some distribution such that the expected stopping time 
is~$T$. 
If the stopping time distribution is not known, then to ensure robustness, the 
distribution is chosen by an adversary, to represent the worst-case scenario.
 
A stationary plan for every vertex always chooses the same outgoing edge.
For fixed horizon or fixed stopping-time distribution, stationary plans are 
not sufficient for optimality. 
Quite surprisingly we show that when an adversary chooses the stopping-time 
distribution with expected stopping time~$T$, then stationary plans are sufficient.
While computing optimal stationary plans for fixed horizon is NP-complete,
we show that computing optimal stationary plans under adversarial stopping-time 
distribution can be achieved in polynomial time.
Consequently, our polynomial-time algorithm for adversarial stopping time also computes 
an optimal plan among all possible plans.
\end{abstract}


\maketitle

\section{Introduction}

\newcolumntype{C}[1]{>{\centering\let\newline\\\arraybackslash\hspace{0pt}}m{#1}}

\noindent{\em Graph search algorithms.}
Reasoning about graphs is a fundamental problem in computer science, 
which is studied widely in logic (such as to describe graph properties
with logic~\cite{GradelBook,CourcelleBook}) and artificial intelligence~\cite{AIBook,LaValle}.
Graph search/planning algorithms are at the heart of such analysis, and 
gives rise to some of the most important algorithmic problems in computer 
science, such as shortest path, travelling salesman problem (TSP), etc.

\smallskip\noindent{\em Finite-horizon planning.} 
A classical problem in graph planning is the {\em finite-horizon} 
planning problem~\cite{LaValle}, where the input is a directed graph with weights 
assigned to every edge and a time horizon $T$. 
The weight of an edge represents the reward/cost of the edge.
A {\em plan} is an infinite path, and for finite horizon $T$ the utility of the plan 
is the sum of the weights of the first $T$ edges. 
An {\em optimal} plan maximizes the utility.
The computational problem for finite-horizon planning is to compute the optimal 
utility and an optimal plan.
The finite-horizon planning problem has many applications: 
the qualitative version of the problem corresponds to finite-horizon reachability,
which plays an important role in logic and verification (e.g., bounded until in RTCTL, 
and bounded model-checking~\cite{EMSS92,BCCSZ03});
and the more general quantitative problem of optimizing the sum of rewards 
has applications in artificial intelligence and robotics~\cite[Chapter~10, Chapter~25]{AIBook},
and in control theory and game theory~\cite[Chapter~2.2]{FV97},~\cite[Chapter~6]{OR94}.

\smallskip\noindent{\em Solutions for finite-horizon planning.} 
For finite-horizon planning the classical solution approach is dynamic programming 
(or Bellman equations), which corresponds to backward induction~\cite{Howard,FV97}. 
This approach not only works for graphs, but also for other models 
(e.g., Markov decision processes~\cite{PT87}). 
A {\em stationary} plan is a path where for every vertex always the same 
choice of edge is made. 
For finite-horizon planning, stationary plans are not sufficient for 
optimality, and in general, optimal plans are quite involved, and 
represented as transducers optimal plans require storage proportional to at least $T$ (see Example~\ref{ex:memory-fixed-T}).
Since in general optimal plans are involved, a related computational question is to compute 
effective simple plans, i.e., plans that are optimal among stationary plans.

\smallskip\noindent{\em Expected finite-horizon planning.} 
A natural variant of the finite-horizon planning problem is to consider 
expected time horizon, instead of the fixed time horizon. 
In the finite-horizon problem the allowed stopping time of the planning problem 
is a Dirac distribution at time $T$. 
In expected finite-horizon problem the {\em expected} stopping time is $T$. 
A well-known example where the fixed finite-horizon and the expected 
finite-horizon problems are fundamentally different is playing Prisoner's 
Dilemma: if the time horizon is fixed, then defection is the only dominant 
strategy, whereas for expected finite-horizon problem cooperation is 
feasible~\cite[Chapter~5]{Nowak}.
Another classical example that is very well-studied is the notion of 
{\em discounting}, where at each time step the stopping probability is 
$\lambda$, and this corresponds to the case that the expected stopping time is 
$1/\lambda$~\cite{FV97}.

\smallskip\noindent{\em Specified vs. adversarial distribution.}
For the expected finite-horizon problem there are two variants: 
(a)~{\em specified distribution:} the stopping-time distribution is specified; 
and (b)~{\em adversarial distribution:} the stopping-time distribution is 
unknown and decided by an adversary. 
The expected finite-horizon problem with adversarial distribution represents 
the robust version of the planning problem, where the distribution is unknown 
and the adversary represents the worst-case scenario.
Thus this problem presents the robust extension of the classical finite-horizon 
planning that has a wide range of applications.

\smallskip\noindent{\em Results.} 
In this work we consider the expected finite-horizon planning problems in graphs. 
To the best of our knowledge this problem has not been studied in the literature.
\begin{itemize}
\item Our first simple result is that for the specified distribution problem, the 
optimal value can be computed in polynomial time (Theorem~\ref{thm:fixed-distribution}).
However, since the specified distribution generalizes the fixed finite-horizon problem,
the optimal plan description as an explicit transducer is of size $T$. 
Hence the output complexity is not polynomial in general.
Second, we consider the decision problem whether there is a stationary plan to ensure 
a given utility. We show that this problem is NP-complete (Theorem~\ref{thm:hard}). 
\end{itemize}
Our most interesting and surprising results are for the adversarial 
distribution problem, which we describe below:
\begin{compactitem}
\item We show that stationary plans suffice for optimality (Theorem~\ref{thm:plans}). 
This result is surprising and counter-intuitive. 
Both in the classical finite-horizon problem and the specified distribution 
problem the adversary does not have any choice, and in both cases stationary 
plans do not suffice for optimality. 
Surprisingly we show that in the presence of an adversary the simpler class of 
stationary plans suffices for optimality.

\item For the expected finite-horizon problem with adversarial distribution, 
the backward induction approach does not work, as there is no a-priori bound 
on the stopping time. 
We develop new algorithmic ideas to show that the optimal value can still be 
solved in polynomial time (Theorem~\ref{thm:algo}).
Moreover, our algorithm also computes and outputs an optimal stationary plan 
in polynomial time.
Note that our algorithm also computes stationary optimal plans (which are as 
well optimal among all plans) in polynomial time, whereas computing 
stationary optimal plans for fixed finite horizon is NP-complete.
\end{compactitem}
Our results are summarized in Table~\ref{tab:complexity} and are 
relevant for synthesis of robust plans for expected finite-horizon planning.
\begin{shortversion}
Detailed proofs are available in the fuller version attached.
\end{shortversion}

\begin{table*}[!t]
\begin{center}
\begin{tabular}{l|C{26mm}C{26mm}|C{18mm}C{18mm}}
     {\Large \strut}          & arbitrary & stationary        & arbitrary & stationary    \\
\hline   
Fixed horizon {\Large \strut} & PTIME     & {\bf NP-complete} & $O(T)$    & $O(\abs{V})$  \\
\hline   
Expected horizon {\Large \strut} &  \multicolumn{2}{c|}{{\bf PTIME}} & \multicolumn{2}{c}{$\mathbf{O(\abs{V})}$} \\
\end{tabular}
\bigskip
\caption{Computational complexity (left) and plan complexity (right).
New results in boldface.\label{tab:complexity}}
\end{center}
\end{table*}

\section{Preliminaries}\label{sec:prelim}

\noindent{\em Weighted graphs.}
A \emph{weighted graph} $G = \tuple{V,E,w}$ consists of a finite set $V$ 
of vertices, a set $E \subseteq V \times V$ of edges, and
a function $w\colon E \to \zed$ that assigns a weight to each edge
of the graph. 

\smallskip\noindent{\em Plans and utilities.}
A {\em plan} is an infinite \emph{path} in $G$ from a vertex $v_0$, that is a sequence $\rho = e_0 e_1 \dots$ of edges
$e_i = (v_i, v'_i) \in E$ such that $v'_i = v_{i+1}$ for all $i \geq 0$. 
A path induces the sequence of utilities $u_0, u_1, \dots$
where $u_i = \sum_{0\leq k \leq i} w(e_k)$ for all $i \geq 0$.
We denote by $U_G$ the set of all sequences of utilities induced by the paths
of $G$. For finite paths $\rho = e_0 e_1 \dots e_k$ (i.e., finite prefixes 
of paths), we denote by $\starts(\rho) = v_0$ and $\ends(\rho) = v'_{k}$
the initial and last vertex of $\rho$, and by $\abs{\rho} = k+1$ the length of $\rho$.

\smallskip\noindent{\em Plans as transducers.}
A plan is described by a {\em transducer} (Mealy machine or Moore machine~\cite{HU79}) 
that given a prefix of the path (i.e., a finite sequence of edges) chooses the next edge. 
A {\em stationary} plan is a path where for every vertex the same choice of edge is made
always. 
A stationary plan as a Mealy machine has one state, and as a Moore machine has at most
$\abs{V}$ states.
Given a graph $G$ we denote by $S_G$ the set of all sequences of utilities induced by stationary plans in $G$.


\smallskip\noindent{\em Distributions and stopping times.}
A \emph{sub-distribution} is a function $\delta \colon \nat \to [0,1]$
such that $p_{\delta} = \sum_{t \in \nat} \delta(t) \in (0,1]$. The value $p_{\delta}$
is the probability mass of $\delta$. Note that $p_{\delta} \neq 0$.
The support of $\delta$ is $\Supp(\delta) = \{t \in \nat \mid \delta(t) \neq 0\}$,
and we say that $\delta$ is the sum of two sub-distributions $\delta_1$ and $\delta_2$,
written $\delta = \delta_1 + \delta_2$, if $\delta(t) = \delta_1(t) + \delta_2(t)$ for all $t \in \nat$.
A \emph{stopping-time distribution} (or simply, a distribution) is a sub-distribution 
with probability mass equal to~$1$. 
We denote by $\Delta$ the set of all stopping-time distributions, and by $\Delta^{\upuparrows}$
the set of all distributions $\delta$ with $\abs{\Supp(\delta)} \leq 2$,
called the bi-Dirac distributions.

\smallskip\noindent{\em Expected utility and expected time.}
The \emph{expected utility} of a sequence $u = u_0, u_1, \dots$ of utilities under a sub-distribution $\delta$
is $\E_{\delta}(u) = \frac{1}{p_{\delta}} \cdot \sum_{t \in \nat} u_t \cdot \delta(t)$.
In particular, the expected utility of the identity sequence $0,1,2,\dots$
is called the \emph{expected time}, denoted by $\E_{\delta}$.

\section{Expected Finite-horizon: Specified Distribution}
Given a stopping-time distribution $\delta$ with finite support, we show that the optimal expected utility 
can be computed in polynomial time. 
This result is straightforward.

\begin{theorem}\label{thm:fixed-distribution}
Let $G$ be a weighted graph. Given a stopping-time distribution $\delta = \{(t_1,p_1),\dots,(t_k,p_k)\} \subseteq \nat \times \rat$,
with all numbers encoded in binary, the optimal expected utility $\sup_{u \in U_G} \E_{\delta}(u)$ 
can be computed in polynomial time.
\end{theorem}

\begin{longversion}
A special case of the problem in Theorem~\ref{thm:fixed-distribution} is the 
fixed-length optimal path problem, which is to find an optimal path (that maximizes the total utility)
of fixed length $T$, corresponding to the distribution $\delta = \{(T,1)\}$.
A pseudo-polynomial time solution is known for this problem, 
based on a value-iteration algorithm~\cite[Section 2.3]{LaValle}. The algorithm
runs in time $O(T \cdot \abs{V}^2)$ (where $T$ is encoded in binary), and relies 
on the following recursive relation, where $A_t(v)$ is the optimal value among 
the paths of length $t$ that start in $v$: 
$$A_t(v) = \max_{v' \in V} \  w(v,v') + A_{t-1}(v').$$

A polynomial algorithm running in $O(\log(T) \cdot \abs{V}^3)$ to obtain
$A_T(v)$ is to compute, in the max-plus algebra\footnote{In the max-plus algebra,
the matrix product $C = A \cdot B$ is defined by $C_{ij} = \max_k A_{ik} + B_{kj}$.}, 
the $T$-th power of the transition matrix $M$ of the weighted 
graph, where $M_{ij} = w(i,j)$ if $(i,j) \in E$, and $M_{ij} = -\infty$ otherwise. 
The power $M^T$ can be computed in time $O(\log(T) \cdot \abs{V}^3)$ 
by successive squaring of $M$ and summing up according to the binary representation
of $T$, which gives a polynomial algorithm to compute~$A_T(v)$ since it is the largest 
element in the column of~$M^T$ corresponding to~$v$ (note that the entries of the matrix
$M^T$ are bounded by $\abs{V} \cdot W$, where $W$ is the largest absolute weight in the graph).
We now present the proof of Theorem~\ref{thm:fixed-distribution}.

\begin{proof}[Proof of Theorem~\ref{thm:fixed-distribution}]
Given the weighted graph $G =  \tuple{V,E,w}$ and the distribution $\delta = \{(t_1,p_1),\dots,(t_k,p_k)\}$,
we reduce the problem to finding an optimal path of length $k$ in 
a layered graph $G'$ where the transitions between layer $i$ and layer $i+1$
mimic sequences of $t_{i+1} - t_i$ transitions in the original graph.
For $t \geq 2$, define the $t$-th power of $E$ recursively by 
$E^t = \{(v_0,v_2) \mid \exists v_1: (v_0,v_1) \in E \land (v_1,v_2) \in E^{t-1} \}$ 
where $E^1 = E$. Let $M$ be the transition matrix of the original weighted graph.
We construct the graph $G' =  \tuple{V',E',w'}$ where
\begin{itemize}
\item $V' = V \times \{0,\dots,k\}$,
\item $E' = \{ (\tuple{v,i},\tuple{v',i+1}) \mid (v,v') \in E^{t_{i+1} - t_i} \land 0 \leq i < k \}$ where $t_0 = -1$, and
\item $w'(\tuple{v,i},\tuple{v',i+1}) = (p_{i+1} + p_{i+2} + \dots + p_{k})\cdot (M^{t_{i+1} - t_i})_{v,v'}$.
\end{itemize}
The optimal expected utility $\sup_{u \in U_G} \E_{\delta}(u)$ is
the same as the optimal fixed-length path value for length $k$ in $G'$.
The correctness of this reduction relies on the fact that 
the probability of not stopping before time $t_{i+1}$ is $p_{i+1} + p_{i+2} + \dots + p_{k}$
and the largest utility of a path of length  $t_{i+1} - t_i$ from $v$ to $v'$
is $(M^{t_{i+1} - t_i})_{v,v'}$.
Given a path $(v_0,v_1)(v_1,v_2) \dots (v_{k-1},v_{k})$ 
of length $k$ in $G'$ (that induces a sequence $w'_0 \dots w'_{k-1}$ of weights),
we can construct a path of length $t_k + 1$ in $G$ (visiting $v_{i}$ at time $t_i$
and inducing a sequence $u$ of utilities), and we show that the value 
of the path of length $k$ in $G'$ is the same as the expected utility of 
the corresponding path in $G$ with stopping time distributed according to $\delta$,
as follows (where $u_{t_0} = 0$):
\begin{align*}
 \sum_{i=0}^{k-1}  w'_i & =
     \sum_{i=0}^{k-1}  \left(\sum_{j=i+1}^{k} p_j\right) \cdot (u_{t_{i+1}} - u_{t_i}) \\
 & = \sum_{j=1}^{k} p_j \cdot \sum_{i=0}^{j-1} (u_{t_{i+1}} - u_{t_{i}})\\
 & = \sum_{j=1}^{k} p_j \cdot u_{t_j}
\end{align*}
Conversely, given an arbitrary path in $G$, let $v_{i}$ be the vertex visited at time $t_i$,
and consider the path $(\tuple{v_0,0},\tuple{v_1,1})(\tuple{v_1,1},\tuple{v_2,2}) \dots (\tuple{v_{k-1},k-1},\tuple{v_{k},k})$
in $G'$, which has a total utility at least the same as the expected utility of the given path in $G$. 

Therefore, the problem can be solved by finding the optimal fixed-length path 
value for length $k$ in $G'$, which can be done in polynomial time (see the
remark after Theorem~\ref{thm:fixed-distribution}).
\end{proof}
\end{longversion}

In the fixed-horizon problem with $\delta=\{(T,1)\}$, the optimal plan need not be stationary.
The example below shows that in general the transducer for optimal plans 
require $O(T/\abs{V})$ states as Mealy machine, and $O(T)$ states as Moore machine. 

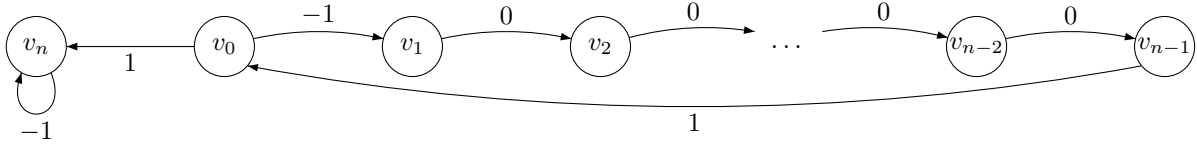
\begin{figure*}[!tb]
  \begin{center}
    \hrule

\begin{picture}(160,24)(0,0)

\gasset{Nw=8,Nh=8,Nmr=4, rdist=1, loopdiam=5}

\node[Nmarks=n, Nframe=y](v0)(30,15){$v_0$}
\node[Nmarks=n, Nframe=y](v1)(55,15){$v_1$}
\node[Nmarks=n, Nframe=y](v2)(80,15){$v_2$}
\node[Nmarks=n, Nframe=n, Nw=10](v3)(105,15){$\dots$}
\node[Nmarks=n, Nframe=y](v4)(130,15){$v_{n-2}$}
\node[Nmarks=n, Nframe=y](v5)(155,15){$v_{n-1}$}

\node[Nmarks=n, Nframe=y](vn)(5,15){$v_n$}

\drawedge[ELpos=50, ELside=l, curvedepth=2](v0,v1){$-1$}
\drawedge[ELpos=50, ELside=l, curvedepth=2](v1,v2){$0$}
\drawedge[ELpos=50, ELside=l, curvedepth=2, eyo=1](v2,v3){$0$}
\drawedge[ELpos=50, ELside=l, curvedepth=2, syo=1](v3,v4){$0$}
\drawedge[ELpos=50, ELside=l, curvedepth=2](v4,v5){$0$}
\drawedge[ELpos=50, ELside=l, curvedepth=6, syo=-2, eyo=-2](v5,v0){$1$}


\drawedge[ELpos=50, ELside=l, curvedepth=0](v0,vn){$1$}
\drawloop[ELside=l,loopCW=y, loopdiam=5, loopangle=270](vn){$-1$}







\end{picture}
    \hrule
      \caption{A weighted graph (with $n+1$ vertices) where the optimal path (of length $T = k \cdot n + 1$)
is not simple: at $v_0$, the optimal plan chooses $k$ times the edge $(v_0,v_1)$, and then the edge $(v_0,v_n)$. \label{fig:cycle-memory}}
  \end{center}
\end{figure*}

\begin{example}\label{ex:memory-fixed-T}
Consider the graph of \figurename~\ref{fig:cycle-memory}
with $\abs{V} = n+1$ vertices, and time bound $T = k \cdot n + 1$ (for some constant $k$).
The optimal plan from $v_0$ is to repeat $k$ times the cycle $v_0,v_1,\dots,v_{n-1}$
and then switch to $v_n$. This path has value $1$, and all other paths have lower
value: if only the cycle $v_0,v_1,\dots,v_{n-1}$ is used, then the value is at most $0$,
and the same holds if the cycle on $v_n$ is ever used before time $T$.
The optimal plan can be represented by a Mealy machine of size $O(T / \abs{V})$
that counts the number of cycle repetitions before switching to $v_n$.
A Moore machine requires size $T$ as it needs a new memory state at every step
of the plan.
\end{example}

\begin{example}\label{ex:memory-fixed-T-multiple-cycles}
In the example of \figurename~\ref{fig:mult-cycle-memory} the optimal plan needs 
to visit several different cycles, not just repeating a single cycle and possible 
switching only at the end. The graph consists of three loops on $v_0$ 
with weights~$0$ and respective length $6$, $10$, and $15$, and an edge to $v_1$
with weight $1$. For expected time $T = 6 + 10 + 15 + 1$, the optimal plan has value $1$ and
needs to stop exactly when reaching $v_1$ (to avoid the negative self-loop on $v_1$).
It is easy to show that the remaining length $T-1 = 31$ can only be obtained by visiting
each cycle once: as $31$ is not an even number, the path has to visit a cycle
of odd length, thus the cycle of length $15$; analogously, as $31$ is not a
multiple of $3$, the path has to visit the cycle of length $10$, etc.
This example can be easily generalized to an arbitrary number of cycles by using
more prime numbers.
\end{example}

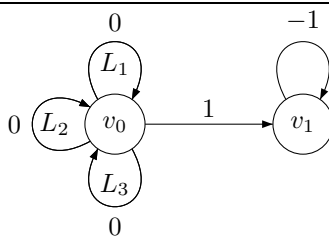
\begin{figure}[!tb]
  \begin{center}
    \hrule

\begin{picture}(40,32)(0,0)

\gasset{Nw=8,Nh=8,Nmr=4, rdist=1, loopdiam=5}

\node[Nmarks=n, Nframe=y](v0)(11,16){$v_0$}
\node[Nmarks=n, Nframe=y](v1)(36,16){$v_1$}

\drawedge[ELpos=50, ELside=l, curvedepth=0](v0,v1){$1$}

\gasset{loopdiam=7, ELdist=1.5}
\drawloop[ELside=l,loopCW=y, loopangle=90](v0){$0$}
\drawloop[ELside=l,loopCW=y, loopangle=180](v0){$0$}
\drawloop[ELside=l,loopCW=y, loopangle=270](v0){$0$}
\drawloop[ELside=r,loopCW=y, loopangle=90](v0){$L_1$}
\drawloop[ELside=r,loopCW=y, ELdist=1, loopangle=180](v0){$L_2$}
\drawloop[ELside=r,loopCW=y, loopangle=270](v0){$L_3$}

\drawloop[ELside=l,loopCW=y, loopangle=90](v1){$-1$}







\end{picture}
    \hrule
      \caption{Three loops of respective length $L_1 = 6 = 2\cdot 3$, $L_2 = 10 = 2\cdot 5$, and $L_3 = 15 = 3 \cdot 5$.
For $T= 32 = 6 + 10 +15 +1$, the optimal plan needs to visit each cycle once. \label{fig:mult-cycle-memory}}
  \end{center}
\end{figure}

We now consider the complexity of computing optimal plans among stationary plans.

\begin{theorem}\label{thm:hard}
Let $G$ be a weighted graph and $\lambda$ be a rational utility threshold. 
Given a stopping-time distribution $\delta$, whether $\sup_{u \in S_G} \E_{\delta}(u) \geq \lambda$
(i.e., whether there is a stationary plan with utility at least $\lambda$) is NP-complete.
The NP-hardness holds for the fixed-horizon problem $\delta=\{(T,1)\}$, even when $T$ 
and all weights are in $O(\abs{V})$, and thus expressed in unary.
\end{theorem}

\begin{proof}
The NP upper bound is easily obtained by guessing a stationary plan (i.e., one 
edge for each vertex of the graph) and checking that the value of the induced
path is at least $\lambda$.

The NP hardness follows from a result of~\cite{FHW80} where, given a directed graph $\G$
and four vertices $w,x,y,z$, the problem of deciding the existence of two (vertex) disjoint 
simple paths (one from $w$ to $x$ and the other from $y$ to $z$) is shown to be NP-complete. 
It easily follows that given a directed graph, and two vertices $v_1, v_2$, the problem of 
deciding the existence of a simple cycle that contains $v_1$ and $v_2$ is NP-complete.
We present a reduction from the latter problem, illustrated in \figurename~\ref{fig:np-hardness}.
We construct a weighted graph from $\G$, by adding two vertices {\sf start} and {\sf sink},
and all edges have weight $0$ except those from $v_2$ with weight $1$, and
the edge $(v_1,{\sf sink})$ with weight $n+1$ where $n$ is the number of vertices in $\G$.
Let $T = n+1$ and the utility threshold $\lambda = n+2$.

If there exists a simple cycle containing $v_1$ and $v_2$ in $\G$, then
there exists a stationary plan from {\sf start} that visits $v_2$ then $v_1$
in at most $n$ steps. This plan can be prolonged to a plan of $n+1$
steps by going to {\sf sink} and using the self-loop. The total weight is $n+2 = \lambda$.

If there is no simple cycle containing $v_1$ and $v_2$ in $\G$, then
no stationary plan can visit first $v_2$ then $v_1$. We show that every stationary plan 
has value at most $n+1 < \lambda$. 
First if a stationary plan uses the edge $(v_1,{\sf sink})$, then
$v_2$ is not visited and all weights are $0$ except 
the weight $n+1$ from $v_1$ to {\sf sink}. 
Otherwise, if a stationary plan does not use the edge $(v_1,{\sf sink})$, then 
all weights are at most $1$, and the total utility is at most $n+1$.
In both cases, the utility is smaller than $\lambda$,
which establishes the correctness of the reduction. 
\end{proof}

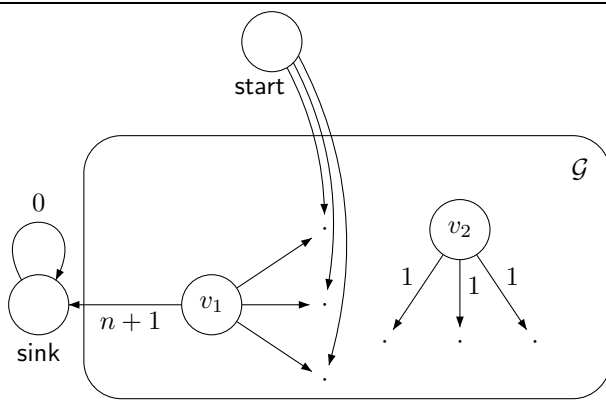
\begin{figure}[!tb]
  \begin{center}
    \hrule

\begin{picture}(80,55)(0,0)

\gasset{Nw=8,Nh=8,Nmr=4, rdist=1, loopdiam=5}

\node[Nmarks=n, Nframe=y, Nh=35, Nw=70, Nmr=5](v0)(45,20){}
\node[Nmarks=n, Nframe=n](v1)(76,33){$\G$}

\node[Nmarks=n, Nframe=y](v1)(27,15){$v_1$}
\node[Nmarks=n, Nframe=y](v2)(60,25){$v_2$}
\node[Nmarks=n, Nframe=y, ExtNL=y, NLangle=255, NLdist=1](start)(35,50){{\sf start}}
\node[Nmarks=n, Nframe=y, ExtNL=y, NLangle=270, NLdist=1](sink)(4,15){{\sf sink}}

\gasset{Nw=4,Nh=4,Nmr=2}

\node[Nmarks=n, Nframe=n](v11)(42,25){$\cdot$}
\node[Nmarks=n, Nframe=n](v12)(42,15){$\cdot$}
\node[Nmarks=n, Nframe=n](v13)(42,5){$\cdot$}

\node[Nmarks=n, Nframe=n](v21)(50,10){$\cdot$}
\node[Nmarks=n, Nframe=n](v22)(60,10){$\cdot$}
\node[Nmarks=n, Nframe=n](v23)(70,10){$\cdot$}

\drawedge[ELpos=50, ELside=l, curvedepth=0](v1,v11){}
\drawedge[ELpos=50, ELside=l, curvedepth=0](v1,v12){}
\drawedge[ELpos=50, ELside=l, curvedepth=0](v1,v13){}

\drawedge[ELpos=50, ELside=r, curvedepth=0](v2,v21){$1$}
\drawedge[ELpos=50, ELside=l, curvedepth=0](v2,v22){$1$}
\drawedge[ELpos=50, ELside=l, curvedepth=0](v2,v23){$1$}

\drawedge[ELpos=47, ELside=l, curvedepth=0](v1,sink){$n+1$}

\drawedge[ELpos=50, ELside=l, curvedepth=2](start,v11){}
\drawedge[ELpos=50, ELside=l, syo=2, curvedepth=4](start,v12){}
\drawedge[ELpos=50, ELside=l, syo=4, curvedepth=6](start,v13){}


\gasset{loopdiam=7, ELdist=1.5}
\drawloop[ELside=l,loopCW=y, loopangle=90](sink){$0$}







\end{picture}
    \hrule
      \caption{The NP-hardness reduction of Theorem~\ref{thm:hard}. \label{fig:np-hardness}}
  \end{center}
\end{figure}

\section{Expected Finite-horizon: Adversarial Distribution}
We now consider the computation of the following \emph{optimal} values under adversarial distribution.
Given a weighted graph~$G$ and an expected stopping time $T \in \rat$, we define the following:
\begin{compactitem}
\item {\em Optimal values of plans.} For a plan $\rho$ that induces the sequence $u$ of utilities,
let 
$$\val(\rho,T)=\val(u,T) = \inf_{{\delta \in \Delta: \E_{\delta} = T}} \  \E_{\delta}(u).$$

\item {\em Optimal value.} The optimal value is the supremum value over all plans:
$$\val(G,T) = \sup_{u \in U_G} \val(u,T).$$
\end{compactitem}
Our two main results are related to the plan complexity and a polynomial-time algorithm.

\begin{theorem}\label{thm:plans}
For all weighted graphs $G$ and for all $T$ we have 
$$
\val(G,T) = \sup_{u \in U_G} \val(u,T) = \sup_{u \in S_G} \val(u,T),$$ 
i.e., optimal stationary plans exist for expected finite-horizon under adversarial distribution.
\end{theorem}

\begin{remark}
Note that in contrast to fixed finite-horizon problem, where stationary plans do not
suffice, we show in the presence of an adversary, the simpler class of stationary plans 
are sufficient for optimality in expected finite-horizon.
Moreover, while optimal plans require $O(T/\abs{V})$-size Mealy (resp., $O(T)$-size Moore)
machines for fixed-length plans, our results show that under adversarial distribution 
optimal plans require $O(1)$-size Mealy (resp., $O(\abs{V})$-size Moore) machines.
\end{remark}

\begin{theorem}\label{thm:algo}
Given a weighted graph $G$ and expected finite-horizon $T$, 
whether $\val(G,T) \geq 0$ can be decided in $O(\abs{V}^{16} \cdot \log (T))$ time, 
and computing $\val(G,T)$ can be done in $O(\abs{V}^{16} \cdot \log(W \cdot \abs{V}) \cdot \log (T))$ time. 
\end{theorem}

\subsection{Theorem~\ref{thm:plans}: Plan Complexity}
In this section we prove Theorem~\ref{thm:plans}. We start with the notion of sub-distributions.
Two sub-distributions $\delta,\delta'$ are \emph{equivalent} if they have the
same probability mass, and the same expected time, that is $p_{\delta} = p_{\delta'}$
and $\E_{\delta} = \E_{\delta'}$. The following result is straightforward.

\begin{lemma}\label{lem:eq-sub-dist}
If $\delta_1,\delta_1'$ are equivalent sub-distributions, and
$\delta_1 + \delta_2$ is a sub-distribution, 
then $\delta_1 + \delta_2$ and $\delta_1' + \delta_2$ are equivalent sub-distributions.
\end{lemma}

\subsubsection{Bi-Dirac distributions are sufficient}\label{sec:Bi-Dirac-sufficient}
By Lemma~\ref{lem:eq-sub-dist}, we can decompose distributions as the
sum of two sub-distributions, and we can replace one of the two sub-distributions 
by a simpler (yet equivalent) one to obtain an equivalent distribution.
We show that, given a sequence $u$ of utilities, 
for all sub-distributions with three points $t_1,t_2,t_3$ in their support
(see \figurename~\ref{fig:timeline}), 
there exists an equivalent sub-distribution with only two points in its support
that gives a lower expected value for $u$. Intuitively,
if one has to distribute a fixed probability mass (say $1$) 
among three points with a fixed expected time $T$, assigning probability $p_i$
at point $t_i$, then we have $p_3 = 1 - p_1 - p_2$ and 
$p_1 \cdot t_1 + p_2 \cdot t_2 + p_3 \cdot t_3 = T$, i.e., 
$$\underbrace{p_1 \cdot (t_1-t_3)}_{p'_1} + \underbrace{p_2 \cdot (t_2-t_3)}_{p'_2} = T - t_3.$$ 
The expected utility is
$$ p_1 \cdot u_{t_1} + p_2 \cdot u_{t_2} + p_3 \cdot u_{t_3} =
 p'_1 \cdot \frac{u_{t_1} - u_{t_3}}{t_1 - t_3} + p'_2 \cdot \frac{u_{t_2} - u_{t_3}}{t_2 - t_3} +  u_{t_3} 
$$
which is a linear expression in variables $\{p'_1,p'_2\}$ where the sum $p'_1 + p'_2$ is constant.
Hence the least expected utility is obtained for either $p'_1 = 0$, or $p'_2 = 0$.
This is the main argument\footnote{This argument works here because $T>t_2$,
which implies that $0 \leq p_2 \leq 1$ when $p_1 = 0$, and vice versa. A symmetric
argument can be used in the case $T < t_2$, to show that then either $p_2= 0$,
or $p_3 = 0$.}
to show that bi-Dirac distributions are sufficient 
to compute the optimal expected value.

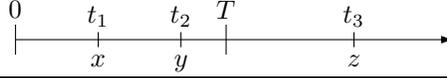
\begin{figure}[!tb]
  \begin{center}
    \hrule

\begin{picture}(62,12)(0,0)

\gasset{Nw=6,Nh=6,Nmr=3, rdist=1, loopdiam=5}

\drawline[AHnb=1,arcradius=1](2,5)(60,5)

\drawline[AHnb=0,arcradius=1](2,3)(2,7)

\drawline[AHnb=0,arcradius=1](13,4)(13,6)
\drawline[AHnb=0,arcradius=1](24,4)(24,6)
\drawline[AHnb=0,arcradius=1](30,3)(30,7)
\drawline[AHnb=0,arcradius=1](47,4)(47,6)

\node[Nmarks=n, Nframe=n](n1)(2,9){$0$}
\node[Nmarks=n, Nframe=n](n1)(13,8){$t_1$}
\nodelabel[ExtNL=y, NLangle=270, NLdist=2](n1){$x$}
\node[Nmarks=n, Nframe=n](n1)(24,8){$t_2$}
\nodelabel[ExtNL=y, NLangle=270, NLdist=2](n1){$y$}
\node[Nmarks=n, Nframe=n](n1)(30,9){$T$}
\node[Nmarks=n, Nframe=n](n1)(47,8){$t_3$}
\nodelabel[ExtNL=y, NLangle=270, NLdist=2](n1){$z$}







\end{picture}
    \hrule
      \caption{Timeline. \label{fig:timeline}}
  \end{center}
\end{figure}

\begin{lemma}[Bi-Dirac distributions are sufficient]\label{lem:Bi-Dirac}
For all sequences $u$ of utilities, for all time bounds $T$, the following holds:
\begin{align*}
&\inf \{\E_{\delta}(u) \mid \delta \in \Delta \land \E_{\delta} = T\} = \\
&\inf \{\E_{\delta}(u) \mid \delta \in \Delta^{\upuparrows} \land \E_{\delta} = T\},
\end{align*}
i.e., the set $\Delta^{\upuparrows}$ of bi-Dirac distributions suffices for the adversary.
\end{lemma}

\begin{longversion}
\begin{proof}
First, we show that for all distributions $\delta \in \Delta$ with $\E_{\delta} = T$,

\begin{itemize}   
\item[$(i)$] there exists an equivalent distribution $\delta' \in \Delta$ such that 
$\abs{\Supp(\delta') \cap [0,T-1]} \leq 1$ and $\E_{\delta'}(u) \leq \E_{\delta}(u)$,
i.e., only one point before $T$ in the support is sufficient, and 

\item[$(ii)$] there exists an equivalent distribution $\delta' \in \Delta$ such that 
$\abs{\Supp(\delta') \cap [T,\infty)} \leq 1$ and $\E_{\delta'}(u) \leq \E_{\delta}(u)$,
i.e., only one point after $T$ in the support is sufficient.
\end{itemize}

The result of the lemma follows from these two claims.

To prove claim $(i)$, first consider an arbitrary \emph{sub-distribution} $\delta$ with 
$\Supp(\delta) = \{t_1, t_2, t_3\}$ where $t_1 < t_2 < t_3$.  
Then $t_1 < \E_{\delta} < t_3$ and either $\E_{\delta} \leq t_2$,
or $t_2 \leq \E_{\delta}$.

We show that among the sub-distributions $\delta'$ equivalent to $\delta$ 
and with $\Supp(\delta') \subseteq \{t_1, t_2, t_3\}$, the smallest expected utility
of $u$ is obtained for $\Supp(\delta') \subsetneq \{t_1, t_2, t_3\}$. 
We present below the argument in the case $t_2 \leq \E_{\delta}$, and show that 
either $\delta'(t_1) = 0$, or $\delta'(t_2) = 0$. A symmetric argument in the
case $\E_{\delta} \leq t_2$ shows that either $\delta'(t_2) = 0$, or $\delta'(t_3) = 0$.

Let $x = \delta'(t_1)$, $y = \delta'(t_2)$, and $z = \delta'(t_3)$.  
Since $\delta'$ and $\delta$ are equivalent, we have
\begin{align*}
& x+y+z = p_{\delta} \\
& x \cdot t_1 + y \cdot t_2 + z \cdot t_3 = p_{\delta} \cdot \E_{\delta}
\end{align*}
Hence
\begin{align*}
& z = p_{\delta} - x - y   \\
& \underbrace{x \cdot (t_1 - t_3)}_{x'} + \underbrace{y \cdot (t_2 - t_3)}_{y'} =  p_{\delta} \cdot (\E_{\delta} - t_3)
\end{align*}
The expected utility of $u$ under $\delta'$ is
\begin{align*}
\E_{\delta'}(u) & = x \cdot u_{t_1} + y \cdot u_{t_2} + z \cdot u_{t_3}   \\
            & = x \cdot (u_{t_1} - u_{t_3}) + y \cdot (u_{t_2} - u_{t_3}) + u_{t_3} \cdot p_{\delta} \\
            & = x' \cdot \frac{u_{t_1} - u_{t_3}}{t_1 - t_3} + y' \cdot \frac{u_{t_2} - u_{t_3}}{t_2 - t_3} + u_{t_3} \cdot p_{\delta} \numberthis \label{eqn:utility}
\end{align*}
Since $x' + y'$ is constant and $x',y' \leq 0$, 
the least value of $\E_{\delta'}(u)$ is obtained either for $x' = 0$ 
(if $\frac{u_{t_1} - u_{t_3}}{t_1 - t_3} \leq \frac{u_{t_2} - u_{t_3}}{t_2 - t_3}$), 
or for $y' = 0$ (otherwise), thus either for $x = 0$, or for $y = 0$.
Note that for $x = 0$, we have $y = \frac{ p_{\delta} \cdot (\E_{\delta} - t_3)}{t_2 - t_3}$
and $z = \frac{ p_{\delta} \cdot (t_2 - \E_{\delta})}{t_2 - t_3}$, which is a feasible solution
as $0 \leq y \leq 1$ and $0 \leq z \leq 1$ since $t_2 \leq \E_{\delta} \leq t_3$, and $0 < p_{\delta} \leq 1$.
Symmetrically, for $y = 0$ we have a feasible solution.

As an intermediate remark, note that for $p_{\delta} = 1$ and $\E_{\delta} = T$, we get (for $y=y'=0$, and symmetrically for $x=x'=0$) 
\begin{equation}
\E_{\delta'}(u) = u_{t_3} + \frac{T - t_3}{t_1 - t_3} \cdot (u_{t_1} - u_{t_3}) 
.\label{eqn:utility-two-points}
\end{equation}

To complete the proof of Claim~$(i)$, given an arbitrary distribution $\delta$
with $\E_{\delta} = T$, we use the above argument to construct 
a distribution equivalent\footnote{Equivalence 
follows from Lemma~\ref{lem:eq-sub-dist}.} to $\delta$ with smaller expected 
utility and one less point in the support. We repeat this argument until
we obtain a distribution $\delta'$ with support that contains at most two points in the 
interval $[0,k]$ where $k$ is such that $\sum_{i \leq k} \delta(i) \cdot i > T-1$.
Such a value of $k$ exists since $\E_{\delta} = \sum_{i \in \nat} \delta(i) \cdot i = T$. 
By the construction of $\delta'$, we have $\sum_{i \leq k} \delta'(i) \cdot i > T-1$
and therefore at most one point in the support of $\delta'$ lies in the
interval $[0,T-1]$, which completes the proof of Claim~$(i)$.

To prove claim~$(ii)$, consider a distribution $\delta$ with $\E_{\delta} = T$, and by claim~$(i)$
we assume that $\delta(t_0) \neq 0$ for some $t_0 < T$, and $\delta(t) = 0$ for all $t < T$ with $t \neq t_0$. 
Let $\nu = \inf_{t \geq T} \frac{u_{t} - u_{t_0}}{t - t_0}$, and we consider two
cases:

\begin{itemize}[leftmargin=*]
\item
  if for all $t \geq T$ such that $t \in \Supp(\delta)$, 
  we have $\frac{u_{t} - u_{t_0}}{t - t_0} = \nu$, then 
  by an analogous of Equation~(\ref{eqn:utility}), we get 
  \begin{align*}
   \E_{\delta}(u) & = u_{t_0} + \sum_{t \geq T} \delta(t) \cdot (t-t_0) \cdot \frac{u_{t} - u_{t_0}}{t - t_0} \\
                & = u_{t_0} + \nu \cdot \sum_{t \geq 0} \delta(t) \cdot (t-t_0)  = u_{t_0} + \nu \cdot (T-t_0) 
  \end{align*}
  which is the expected utility of $u$ under a bi-Dirac distribution 
  with support $\{t_0,t\}$ where $t \geq T$ is any element of $\Supp(\delta)$
  (see Equation~(\ref{eqn:utility-two-points}));

\smallskip
\item 
  otherwise there exists $t \geq T$ such that $t \in \Supp(\delta)$ and 
  $\frac{u_{t} - u_{t_0}}{t - t_0} > \nu$.   
  By an analogous of Equation~(\ref{eqn:utility}), we have
  \begin{align*}
   & \E_{\delta}(u) - u_{t_0} = \sum_{t \geq T} \delta(t) \cdot (t-t_0) \cdot \frac{u_{t} - u_{t_0}}{t - t_0} \\
   & \text{where }  \sum_{t \geq T} \delta(t) \cdot (t-t_0) = T - t_0,
  \end{align*}
  that is $\frac{\E_{\delta}(u) - u_{t_0}}{T - t_0}$ is a convex combination of elements 
  greater than or equal to $\nu$, among which one is greater than $\nu$.
  It follows that $\frac{\E_{\delta}(u) - u_{t_0}}{T - t_0}  > \nu$, and thus there exists
  $\epsilon > 0$ such that $\frac{\E_{\delta}(u) - u_{t_0}}{T - t_0}  > \nu + \epsilon$.

  Consider $t_1$ such that $\frac{u_{t_1} - u_{t_0}}{t_1 - t_0} < \nu + \epsilon$ (which exists
  by definition of~$\nu$),
  and let $\delta'$ be the bi-Dirac distribution $\delta'$ with support $\{t_0,t_1\}$
  and expected time $T$.
  By an analogous of Equation~(\ref{eqn:utility-two-points}), we have 
  \begin{align*} 
   \E_{\delta'}(u) - u_{t_0}  & =  \frac{T - t_0}{t_1 - t_0} \cdot (u_{t_1} - u_{t_0}) \\
                              & < (T - t_0)\cdot(\nu + \epsilon) < \E_{\delta}(u) - u_{t_0} 
  \end{align*}
  Therefore, $\E_{\delta'}(u) < \E_{\delta}(u)$ which concludes the proof since 
  $\delta'$ is a bi-Dirac distribution with $\E_{\delta'} = T$.
\end{itemize}
\end{proof}
\end{longversion}

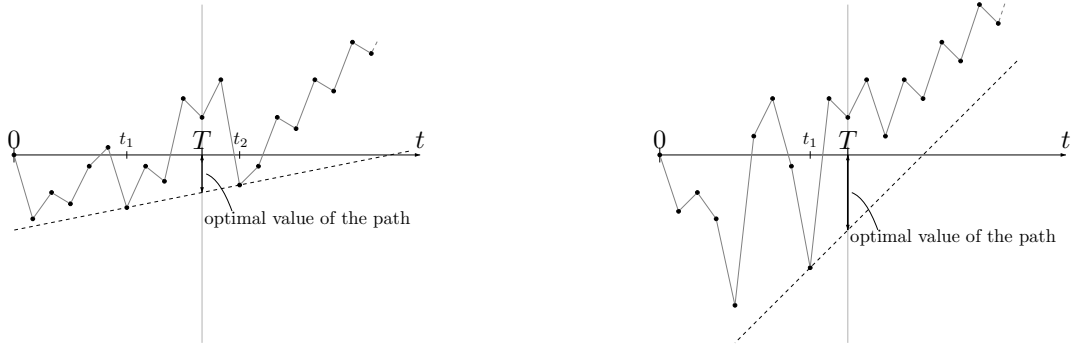
\begin{figure*}
  \hrule
  \begin{minipage}[b]{.48\linewidth}
     \centering 

\scalebox{0.5}{
\begin{picture}(112,100)(0,0)

\gasset{Nw=6,Nh=6,Nmr=3, rdist=1, loopdiam=5}

\drawline[AHnb=0,arcradius=1, linegray=.7](52,5)(52,95)

\drawline[AHnb=1,arcradius=1](2,55)(110,55)

\drawline[AHnb=0,arcradius=1](2,53)(2,57)
\drawline[AHnb=0,arcradius=1](52,53)(52,57)

\drawline[AHnb=0,arcradius=1](32,54)(32,56)
\drawline[AHnb=0,arcradius=1](62,54)(62,56)

\node[Nmarks=n, Nframe=n](n1)(2,59){\scalebox{2}{$0$}}
\node[Nmarks=n, Nframe=n](n1)(52,59){\scalebox{2}{$T$}}
\node[Nmarks=n, Nframe=n](n1)(32,59){\scalebox{1.4}{$t_1$}}
\node[Nmarks=n, Nframe=n](n1)(62.3,59){\scalebox{1.4}{$t_2$}}
\node[Nmarks=n, Nframe=n](n1)(110,59){\scalebox{2}{$t$}}

\drawline[AHnb=0,arcradius=1, dash={1}0 ](2,35)(107,56)
\drawline[AHnb=1,ATnb=1,arcradius=1, linewidth=.35](52,45)(52,55)

\node[Nmarks=n, Nframe=n, Nh=0, Nw=0, Nmr=0](start)(53,50){}
\node[Nmarks=n, Nframe=n, Nh=0, Nw=0, Nmr=0](end)(60,40){}
\node[Nmarks=n, Nframe=n, Nh=0, Nw=0, Nmr=0](label)(80,38){\scalebox{1.4}{optimal value of the path}}

\drawqbpedge[AHnb=0](start,340,end,110){}

\node[Nframe=y, Nh=1,Nw=1,Nmr=.5, Nfill=y](n1)(2,55){}
\node[Nframe=y, Nh=1,Nw=1,Nmr=.5, Nfill=y](n2)(7,38){}
\node[Nframe=y, Nh=1,Nw=1,Nmr=.5, Nfill=y](n3)(12,45){}
\node[Nframe=y, Nh=1,Nw=1,Nmr=.5, Nfill=y](n4)(17,42){}
\node[Nframe=y, Nh=1,Nw=1,Nmr=.5, Nfill=y](n5)(22,52){}
\node[Nframe=y, Nh=1,Nw=1,Nmr=.5, Nfill=y](n6)(27,57){}
\node[Nframe=y, Nh=1,Nw=1,Nmr=.5, Nfill=y](n7)(32,41){}
\node[Nframe=y, Nh=1,Nw=1,Nmr=.5, Nfill=y](n8)(37,52){}
\node[Nframe=y, Nh=1,Nw=1,Nmr=.5, Nfill=y](n9)(42,48){}
\node[Nframe=y, Nh=1,Nw=1,Nmr=.5, Nfill=y](n10)(47,70){}
\node[Nframe=y, Nh=1,Nw=1,Nmr=.5, Nfill=y](n11)(52,65){} 
\node[Nframe=y, Nh=1,Nw=1,Nmr=.5, Nfill=y](n12)(57,75){}
\node[Nframe=y, Nh=1,Nw=1,Nmr=.5, Nfill=y](n13)(62,47){} 
\node[Nframe=y, Nh=1,Nw=1,Nmr=.5, Nfill=y](n14)(67,52){}
\node[Nframe=y, Nh=1,Nw=1,Nmr=.5, Nfill=y](n15)(72,65){}
\node[Nframe=y, Nh=1,Nw=1,Nmr=.5, Nfill=y](n16)(77,62){}
\node[Nframe=y, Nh=1,Nw=1,Nmr=.5, Nfill=y](n17)(82,75){}
\node[Nframe=y, Nh=1,Nw=1,Nmr=.5, Nfill=y](n18)(87,72){}
\node[Nframe=y, Nh=1,Nw=1,Nmr=.5, Nfill=y](n19)(92,85){}
\node[Nframe=y, Nh=1,Nw=1,Nmr=.5, Nfill=y](n20)(97,82){}
\node[Nframe=n, Nh=1,Nw=1,Nmr=.5](n21)(99,86){}


\drawedge[ELpos=50, ELside=r, curvedepth=0, AHnb=0, linegray=.5](n1,n2){}
\drawedge[ELpos=50, ELside=r, curvedepth=0, AHnb=0, linegray=.5](n2,n3){}
\drawedge[ELpos=50, ELside=r, curvedepth=0, AHnb=0, linegray=.5](n3,n4){}
\drawedge[ELpos=50, ELside=r, curvedepth=0, AHnb=0, linegray=.5](n4,n5){}
\drawedge[ELpos=50, ELside=r, curvedepth=0, AHnb=0, linegray=.5](n5,n6){}
\drawedge[ELpos=50, ELside=r, curvedepth=0, AHnb=0, linegray=.5](n6,n7){}
\drawedge[ELpos=50, ELside=r, curvedepth=0, AHnb=0, linegray=.5](n7,n8){}
\drawedge[ELpos=50, ELside=r, curvedepth=0, AHnb=0, linegray=.5](n8,n9){}
\drawedge[ELpos=50, ELside=r, curvedepth=0, AHnb=0, linegray=.5](n9,n10){}
\drawedge[ELpos=50, ELside=r, curvedepth=0, AHnb=0, linegray=.5](n10,n11){}
\drawedge[ELpos=50, ELside=r, curvedepth=0, AHnb=0, linegray=.5](n11,n12){}
\drawedge[ELpos=50, ELside=r, curvedepth=0, AHnb=0, linegray=.5](n12,n13){}
\drawedge[ELpos=50, ELside=r, curvedepth=0, AHnb=0, linegray=.5](n13,n14){}
\drawedge[ELpos=50, ELside=r, curvedepth=0, AHnb=0, linegray=.5](n14,n15){}
\drawedge[ELpos=50, ELside=r, curvedepth=0, AHnb=0, linegray=.5](n15,n16){}
\drawedge[ELpos=50, ELside=r, curvedepth=0, AHnb=0, linegray=.5](n16,n17){}
\drawedge[ELpos=50, ELside=r, curvedepth=0, AHnb=0, linegray=.5](n17,n18){}
\drawedge[ELpos=50, ELside=r, curvedepth=0, AHnb=0, linegray=.5](n18,n19){}
\drawedge[ELpos=50, ELside=r, curvedepth=0, AHnb=0, linegray=.5](n19,n20){}
\drawedge[ELpos=50, ELside=r, curvedepth=0, AHnb=0, linegray=.5, dash={1}0](n20,n21){}








\end{picture}
}
     \subcaption{When an optimal distribution exists {\huge \strut}}\label{fig:geometry1}
  \end{minipage}%
  \hfill
  \begin{minipage}[b]{.48\linewidth}
     \centering 

\scalebox{0.5}{
\begin{picture}(112,100)(0,0)

\gasset{Nw=6,Nh=6,Nmr=3, rdist=1, loopdiam=5}

\drawline[AHnb=0,arcradius=1, linegray=.7](52,5)(52,95)

\drawline[AHnb=1,arcradius=1](2,55)(110,55)

\drawline[AHnb=0,arcradius=1](2,53)(2,57)
\drawline[AHnb=0,arcradius=1](52,53)(52,57)
\drawline[AHnb=0,arcradius=1](42,54)(42,56)

\node[Nmarks=n, Nframe=n](n1)(2,59){\scalebox{2}{$0$}}
\node[Nmarks=n, Nframe=n](n1)(52,59){\scalebox{2}{$T$}}
\node[Nmarks=n, Nframe=n](n1)(42,59){\scalebox{1.4}{$t_1$}}
\node[Nmarks=n, Nframe=n](n1)(110,59){\scalebox{2}{$t$}}

\drawline[AHnb=0,arcradius=1, dash={1}0 ](22,5)(97,80)
\drawline[AHnb=1,ATnb=1,arcradius=1, linewidth=.45](52,35)(52,55)

\node[Nmarks=n, Nframe=n, Nh=0, Nw=0, Nmr=0](start)(53,45){}
\node[Nmarks=n, Nframe=n, Nh=0, Nw=0, Nmr=0](end)(60,35){}
\node[Nmarks=n, Nframe=n, Nh=0, Nw=0, Nmr=0](label)(80,33){\scalebox{1.4}{optimal value of the path}}

\drawqbpedge[AHnb=0](start,340,end,110){}

\node[Nframe=y, Nh=1,Nw=1,Nmr=.5, Nfill=y](n1)(2,55){}
\node[Nframe=y, Nh=1,Nw=1,Nmr=.5, Nfill=y](n2)(7,40){}
\node[Nframe=y, Nh=1,Nw=1,Nmr=.5, Nfill=y](n3)(12,45){}
\node[Nframe=y, Nh=1,Nw=1,Nmr=.5, Nfill=y](n4)(17,38){}
\node[Nframe=y, Nh=1,Nw=1,Nmr=.5, Nfill=y](n5)(22,15){}
\node[Nframe=y, Nh=1,Nw=1,Nmr=.5, Nfill=y](n6)(27,60){}
\node[Nframe=y, Nh=1,Nw=1,Nmr=.5, Nfill=y](n7)(32,70){}
\node[Nframe=y, Nh=1,Nw=1,Nmr=.5, Nfill=y](n8)(37,52){}
\node[Nframe=y, Nh=1,Nw=1,Nmr=.5, Nfill=y](n9)(42,25){} 
\node[Nframe=y, Nh=1,Nw=1,Nmr=.5, Nfill=y](n10)(47,70){}
\node[Nframe=y, Nh=1,Nw=1,Nmr=.5, Nfill=y](n11)(52,65){} 
\node[Nframe=y, Nh=1,Nw=1,Nmr=.5, Nfill=y](n12)(57,75){}
\node[Nframe=y, Nh=1,Nw=1,Nmr=.5, Nfill=y](n13)(62,60){}
\node[Nframe=y, Nh=1,Nw=1,Nmr=.5, Nfill=y](n14)(67,75){}
\node[Nframe=y, Nh=1,Nw=1,Nmr=.5, Nfill=y](n15)(72,70){}
\node[Nframe=y, Nh=1,Nw=1,Nmr=.5, Nfill=y](n16)(77,85){}
\node[Nframe=y, Nh=1,Nw=1,Nmr=.5, Nfill=y](n17)(82,80){}
\node[Nframe=y, Nh=1,Nw=1,Nmr=.5, Nfill=y](n18)(87,95){}
\node[Nframe=y, Nh=1,Nw=1,Nmr=.5, Nfill=y](n19)(92,90){}
\node[Nframe=n, Nh=1,Nw=1,Nmr=.5](n20)(94,96){}

\drawedge[ELpos=50, ELside=r, curvedepth=0, AHnb=0, linegray=.5](n1,n2){}
\drawedge[ELpos=50, ELside=r, curvedepth=0, AHnb=0, linegray=.5](n2,n3){}
\drawedge[ELpos=50, ELside=r, curvedepth=0, AHnb=0, linegray=.5](n3,n4){}
\drawedge[ELpos=50, ELside=r, curvedepth=0, AHnb=0, linegray=.5](n4,n5){}
\drawedge[ELpos=50, ELside=r, curvedepth=0, AHnb=0, linegray=.5](n5,n6){}
\drawedge[ELpos=50, ELside=r, curvedepth=0, AHnb=0, linegray=.5](n6,n7){}
\drawedge[ELpos=50, ELside=r, curvedepth=0, AHnb=0, linegray=.5](n7,n8){}
\drawedge[ELpos=50, ELside=r, curvedepth=0, AHnb=0, linegray=.5](n8,n9){}
\drawedge[ELpos=50, ELside=r, curvedepth=0, AHnb=0, linegray=.5](n9,n10){}
\drawedge[ELpos=50, ELside=r, curvedepth=0, AHnb=0, linegray=.5](n10,n11){}
\drawedge[ELpos=50, ELside=r, curvedepth=0, AHnb=0, linegray=.5](n11,n12){}
\drawedge[ELpos=50, ELside=r, curvedepth=0, AHnb=0, linegray=.5](n12,n13){}
\drawedge[ELpos=50, ELside=r, curvedepth=0, AHnb=0, linegray=.5](n13,n14){}
\drawedge[ELpos=50, ELside=r, curvedepth=0, AHnb=0, linegray=.5](n14,n15){}
\drawedge[ELpos=50, ELside=r, curvedepth=0, AHnb=0, linegray=.5](n15,n16){}
\drawedge[ELpos=50, ELside=r, curvedepth=0, AHnb=0, linegray=.5](n16,n17){}
\drawedge[ELpos=50, ELside=r, curvedepth=0, AHnb=0, linegray=.5](n17,n18){}
\drawedge[ELpos=50, ELside=r, curvedepth=0, AHnb=0, linegray=.5](n18,n19){}
\drawedge[ELpos=50, ELside=r, curvedepth=0, AHnb=0, linegray=.5, dash={1}0](n19,n20){}








\end{picture}
}
     \subcaption{When no optimal distribution exists {\huge \strut}}\label{fig:geometry2}
  \end{minipage}
  \hrule
  \caption{Geometric interpretation of the value of a path.}\label{fig:geometry-all}
\end{figure*}

\subsubsection{Geometric interpretation}\label{sec:Geometric-Interpretation}
It follows from the proof of Lemma~\ref{lem:Bi-Dirac} 
\begin{longversion} 
(and Equation~(\ref{eqn:utility-two-points})) 
\end{longversion}
that the value of the expected utility of a sequence $u$ of utilities
under a bi-Dirac distribution with support $\{t_1,t_2\}$ (where $t_1 < T < t_2$)
and expected time $T$ is 
$$ u_{t_1} + \frac{T - t_1}{t_2 - t_1} \cdot (u_{t_2} - u_{t_1}). $$
In \figurename~\ref{fig:geometry1}, this value is obtained as the
intersection of the vertical axis at $T$ and the line that connects
the two points $(t_1, u_{t_1})$ and $(t_2, u_{t_2})$. Intuitively,
the optimal value of a path is obtained by choosing the two points $t_1$ and $t_2$
such that the connecting line intersects the vertical axis at $T$
as down as possible. 

\begin{lemma}\label{lem:lower-line}
For all sequences $u$ of utilities, if $u_t \geq a \cdot t + b$ for all $t \geq 0$,
then the value of the sequence $u$ is at least $a \cdot T + b$.
\end{lemma}

\begin{proof}
By Lemma~\ref{lem:Bi-Dirac}, it is sufficient to consider bi-Dirac
distributions, and for all bi-Dirac distributions $\delta$ with 
arbitrary support $\{t_1,t_2\}$ the value of $u$ under $\delta$ is
\begin{align*} 
      & \, u_{t_1} + \frac{T - t_1}{t_2 - t_1} \cdot (u_{t_2} - u_{t_1}) \\
 =    & \, \frac{u_{t_1} \cdot (t_2 - T) + u_{t_2} \cdot (T-t_1)}{t_2 - t_1} \\
 \geq & \, \frac{(a \cdot t_1 + b) \cdot (t_2 - T) + (a \cdot t_2 + b) \cdot (T-t_1)}{t_2 - t_1} \\
 \geq & \,\, a \cdot T + b
\end{align*}
\end{proof}

It is always possible to fix an optimal value of $t_1$ (because $t_1 \leq T$ is
to be chosen among a finite set of points), but the optimal value of $t_2$ may 
not exist, as in \figurename~\ref{fig:geometry2}. 
The value of the path is then obtained as $t_2 \to \infty$. 
In general, there exists $t_1 \leq T$ such that it is sufficient 
to consider bi-Dirac distributions with support containing $t_1$ to compute
the optimal value. We say that $t_1$ is a \emph{left-minimizer} of the expected value in the path.
Given such a value of $t_1$, let $\nu = \inf_{t_2 \geq T}  \frac{u_{t_2} - u_{t_1}}{t_2 - t_1}$, and we show in Lemma~\ref{lem:geometric-interpretation}
that $ u_t \geq  u_{t_1} + (t - t_1) \cdot \nu$, for all $t \geq 0$.
This motivates the following definition.

\smallskip\noindent{\em Line of equation $f_u(t)$.} 
Given a left-minimizer $t_1$, we define the line of equation $f_u(t)$ as 
follows:
$$ f_u(t) = u_{t_1} + (t - t_1) \cdot \nu. $$
Note that the optimal expected utility is
$$ \min_{0 \leq t_1 \leq T} \,\, \inf_{t_2 \geq T} \,\,
u_{t_1} + \frac{T - t_1}{t_2 - t_1} \cdot (u_{t_2} - u_{t_1}) = 
\min_{0 \leq t_1 \leq T} \,\, u_{t_1} + (T - t_1) \cdot \nu = 
f_u(T).
$$
In other words, $f_u(T)$ is the optimal value.

\begin{lemma}[Geometric interpretation]\label{lem:geometric-interpretation}
For all sequences $u$ of utilities, we have $u_t \geq f_u(t)$ for all $t \geq 0$, 
and the expected value of $u$ is $f_u(T)$.
\end{lemma}

\begin{longversion}
\begin{proof}
The result holds by definition of $\nu$ for all $t \geq T$. 
For $t < T$, assume towards contradiction that $u_t < u_{t_1} + (t - t_1) \cdot \nu$.
Let $\varepsilon > 0$ such that $u_t = u_{t_1} + (t - t_1) \cdot \nu - \varepsilon$. 
We obtain a contradiction by showing that there exists a bi-Dirac distribution 
under which the expected value of $u$ is smaller than the optimal value of $u$.
Consider a bi-Dirac distribution with support $\{t,t_2\}$ where the value
$t_2$ is defined later.

We need to show that 
$$ u_{t} + \frac{T - t}{t_2 - t} \cdot (u_{t_2} - u_{t}) < 
    u_{t_1} + (T - t_1) \cdot \nu, $$
that is
$$ \frac{u_{t} \cdot (t_2 - T) + u_{t_2} \cdot (T-t)}{t_2 - t} < 
    u_{t_1} + (T - t_1) \cdot \nu $$
which, since $u_t = u_{t_1} + (t - t_1) \cdot \nu - \varepsilon$, holds if (successively)

$$
\begin{array}{l}
u_{t_1} \cdot (t_2 - T) + (t - t_1) \cdot (t_2 - T) \cdot\nu + u_{t_2} \cdot (T-t) \leq \\[1pt]
\multicolumn{1}{r}{\varepsilon \cdot(t_2 - T) + u_{t_1} \cdot (t_2 - t) + (t_2 - t) \cdot (T - t_1) \cdot \nu} \\[6pt]
u_{t_1} \cdot (t-T) + u_{t_2} \cdot (T-t) \leq \\[1pt]
\multicolumn{1}{r}{\varepsilon \cdot(t_2 - T) - \nu \cdot (t \cdot t_2 + t_1 \cdot T - t_2 \cdot T - t \cdot t_1)} \\[6pt]
(u_{t_2} - u_{t_1}) \cdot (T-t) + \nu \cdot (t_2 - t_1) \cdot (t-T) \leq \\[1pt]
\multicolumn{1}{r}{\varepsilon \cdot(t_2 - T)}   \\[6pt]
(T-t) \cdot \left(\frac{u_{t_2} - u_{t_1}}{t_2 - t_1} - \nu\right) \cdot (t_2 - t_1) \leq \varepsilon \cdot(t_2 - T)  
\end{array}
$$



We consider two cases:
$(i)$ if the infimum $\nu$ is attained, then we have $\nu = \frac{u_{t_2} - u_{t_1}}{t_2 - t_1}$
for some $t_2 \geq T$, and the inequality holds; $(ii)$ otherwise, 
we can choose $t_2$ arbitrarily, and large enough to ensure that 
$(T-t) \cdot \left(\frac{u_{t_2} - u_{t_1}}{t_2 - t_1} - \nu\right) $ is smaller than $\frac{\varepsilon}{2}$,
so that the inequality holds.
\end{proof}
\end{longversion}

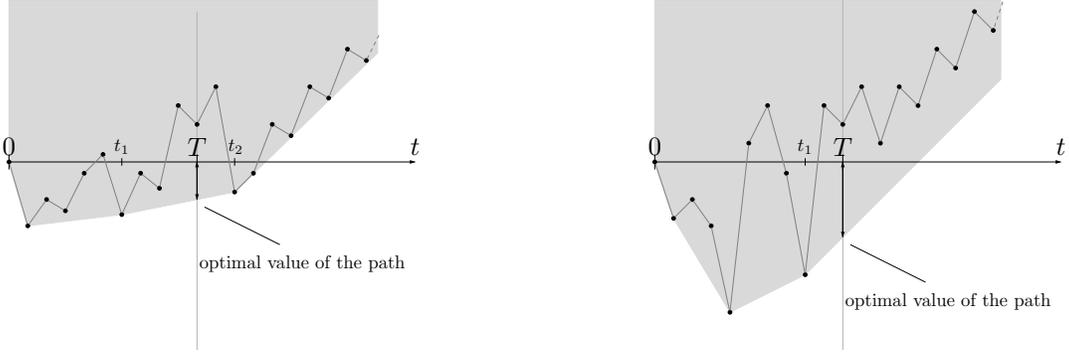
\begin{figure*}
  \hrule
  \begin{minipage}[b]{.48\linewidth}
     \centering 

\scalebox{0.5}{
\begin{picture}(112,95)(0,5)

\gasset{Nw=6,Nh=6,Nmr=3, rdist=1, loopdiam=5}

\drawpolygon[AHnb=0,arcradius=0, linewidth=0, linegray=.85, fillgray=.85](2,98)(2,55)(7,38)(32,41)(62,47)(100,84)(100,98)

\drawline[AHnb=0,arcradius=1, linegray=.7](52,5)(52,95)

\drawline[AHnb=1,arcradius=1](2,55)(110,55)

\drawline[AHnb=0,arcradius=1](2,53)(2,57)
\drawline[AHnb=0,arcradius=1](52,53)(52,57)

\drawline[AHnb=0,arcradius=1](32,54)(32,56)
\drawline[AHnb=0,arcradius=1](62,54)(62,56)

\node[Nmarks=n, Nframe=n](n1)(2,59){\scalebox{2}{$0$}}
\node[Nmarks=n, Nframe=n](n1)(52,59){\scalebox{2}{$T$}}
\node[Nmarks=n, Nframe=n](n1)(32,59){\scalebox{1.4}{$t_1$}}
\node[Nmarks=n, Nframe=n](n1)(62.3,59){\scalebox{1.4}{$t_2$}}
\node[Nmarks=n, Nframe=n](n1)(110,59){\scalebox{2}{$t$}}

\drawline[AHnb=1,ATnb=1,arcradius=1, linewidth=.35](52,45)(52,55)

\drawline[AHnb=0,arcradius=1](54,43)(74,33)
\node[Nmarks=n, Nframe=n, Nh=0, Nw=0, Nmr=0](label)(80,28){\scalebox{1.4}{optimal value of the path}}

\node[Nframe=y, Nh=1,Nw=1,Nmr=.5, Nfill=y](n1)(2,55){}
\node[Nframe=y, Nh=1,Nw=1,Nmr=.5, Nfill=y](n2)(7,38){}
\node[Nframe=y, Nh=1,Nw=1,Nmr=.5, Nfill=y](n3)(12,45){}
\node[Nframe=y, Nh=1,Nw=1,Nmr=.5, Nfill=y](n4)(17,42){}
\node[Nframe=y, Nh=1,Nw=1,Nmr=.5, Nfill=y](n5)(22,52){}
\node[Nframe=y, Nh=1,Nw=1,Nmr=.5, Nfill=y](n6)(27,57){}
\node[Nframe=y, Nh=1,Nw=1,Nmr=.5, Nfill=y](n7)(32,41){}
\node[Nframe=y, Nh=1,Nw=1,Nmr=.5, Nfill=y](n8)(37,52){}
\node[Nframe=y, Nh=1,Nw=1,Nmr=.5, Nfill=y](n9)(42,48){}
\node[Nframe=y, Nh=1,Nw=1,Nmr=.5, Nfill=y](n10)(47,70){}
\node[Nframe=y, Nh=1,Nw=1,Nmr=.5, Nfill=y](n11)(52,65){} 
\node[Nframe=y, Nh=1,Nw=1,Nmr=.5, Nfill=y](n12)(57,75){}
\node[Nframe=y, Nh=1,Nw=1,Nmr=.5, Nfill=y](n13)(62,47){} 
\node[Nframe=y, Nh=1,Nw=1,Nmr=.5, Nfill=y](n14)(67,52){}
\node[Nframe=y, Nh=1,Nw=1,Nmr=.5, Nfill=y](n15)(72,65){}
\node[Nframe=y, Nh=1,Nw=1,Nmr=.5, Nfill=y](n16)(77,62){}
\node[Nframe=y, Nh=1,Nw=1,Nmr=.5, Nfill=y](n17)(82,75){}
\node[Nframe=y, Nh=1,Nw=1,Nmr=.5, Nfill=y](n18)(87,72){}
\node[Nframe=y, Nh=1,Nw=1,Nmr=.5, Nfill=y](n19)(92,85){}
\node[Nframe=y, Nh=1,Nw=1,Nmr=.5, Nfill=y](n20)(97,82){}
\node[Nframe=n, Nh=1,Nw=1,Nmr=.5](n21)(101,90){}


\drawedge[ELpos=50, ELside=r, curvedepth=0, AHnb=0, linegray=.5](n1,n2){}
\drawedge[ELpos=50, ELside=r, curvedepth=0, AHnb=0, linegray=.5](n2,n3){}
\drawedge[ELpos=50, ELside=r, curvedepth=0, AHnb=0, linegray=.5](n3,n4){}
\drawedge[ELpos=50, ELside=r, curvedepth=0, AHnb=0, linegray=.5](n4,n5){}
\drawedge[ELpos=50, ELside=r, curvedepth=0, AHnb=0, linegray=.5](n5,n6){}
\drawedge[ELpos=50, ELside=r, curvedepth=0, AHnb=0, linegray=.5](n6,n7){}
\drawedge[ELpos=50, ELside=r, curvedepth=0, AHnb=0, linegray=.5](n7,n8){}
\drawedge[ELpos=50, ELside=r, curvedepth=0, AHnb=0, linegray=.5](n8,n9){}
\drawedge[ELpos=50, ELside=r, curvedepth=0, AHnb=0, linegray=.5](n9,n10){}
\drawedge[ELpos=50, ELside=r, curvedepth=0, AHnb=0, linegray=.5](n10,n11){}
\drawedge[ELpos=50, ELside=r, curvedepth=0, AHnb=0, linegray=.5](n11,n12){}
\drawedge[ELpos=50, ELside=r, curvedepth=0, AHnb=0, linegray=.5](n12,n13){}
\drawedge[ELpos=50, ELside=r, curvedepth=0, AHnb=0, linegray=.5](n13,n14){}
\drawedge[ELpos=50, ELside=r, curvedepth=0, AHnb=0, linegray=.5](n14,n15){}
\drawedge[ELpos=50, ELside=r, curvedepth=0, AHnb=0, linegray=.5](n15,n16){}
\drawedge[ELpos=50, ELside=r, curvedepth=0, AHnb=0, linegray=.5](n16,n17){}
\drawedge[ELpos=50, ELside=r, curvedepth=0, AHnb=0, linegray=.5](n17,n18){}
\drawedge[ELpos=50, ELside=r, curvedepth=0, AHnb=0, linegray=.5](n18,n19){}
\drawedge[ELpos=50, ELside=r, curvedepth=0, AHnb=0, linegray=.5](n19,n20){}
\drawedge[ELpos=50, ELside=r, curvedepth=0, AHnb=0, linegray=.5, dash={1}0](n20,n21){}








\end{picture}
}
     \subcaption{For the example of \figurename~\ref{fig:geometry1}. {\huge \strut}}\label{fig:ch1}
  \end{minipage}%
  \hfill
  \begin{minipage}[b]{.48\linewidth}
     \centering 

\scalebox{0.5}{
\begin{picture}(112,95)(0,5)

\gasset{Nw=6,Nh=6,Nmr=3, rdist=1, loopdiam=5}

\drawpolygon[AHnb=0,arcradius=0, linewidth=0, linegray=.85, fillgray=.85](2,98)(2,55)(7,40)(22,15)(42,25)(94,77)(94,98)

\drawline[AHnb=0,arcradius=1, linegray=.7](52,5)(52,98)

\drawline[AHnb=1,arcradius=1](2,55)(110,55)

\drawline[AHnb=0,arcradius=1](2,53)(2,57)
\drawline[AHnb=0,arcradius=1](52,53)(52,57)
\drawline[AHnb=0,arcradius=1](42,54)(42,56)

\node[Nmarks=n, Nframe=n](n1)(2,59){\scalebox{2}{$0$}}
\node[Nmarks=n, Nframe=n](n1)(52,59){\scalebox{2}{$T$}}
\node[Nmarks=n, Nframe=n](n1)(42,59){\scalebox{1.4}{$t_1$}}
\node[Nmarks=n, Nframe=n](n1)(110,59){\scalebox{2}{$t$}}

\drawline[AHnb=1,ATnb=1,arcradius=1, linewidth=.35](52,35)(52,55)

\drawline[AHnb=0,arcradius=1](54,33)(74,23)
\node[Nmarks=n, Nframe=n, Nh=0, Nw=0, Nmr=0](label)(80,18){\scalebox{1.4}{optimal value of the path}}

\node[Nframe=y, Nh=1,Nw=1,Nmr=.5, Nfill=y](n1)(2,55){}
\node[Nframe=y, Nh=1,Nw=1,Nmr=.5, Nfill=y](n2)(7,40){}
\node[Nframe=y, Nh=1,Nw=1,Nmr=.5, Nfill=y](n3)(12,45){}
\node[Nframe=y, Nh=1,Nw=1,Nmr=.5, Nfill=y](n4)(17,38){}
\node[Nframe=y, Nh=1,Nw=1,Nmr=.5, Nfill=y](n5)(22,15){}
\node[Nframe=y, Nh=1,Nw=1,Nmr=.5, Nfill=y](n6)(27,60){}
\node[Nframe=y, Nh=1,Nw=1,Nmr=.5, Nfill=y](n7)(32,70){}
\node[Nframe=y, Nh=1,Nw=1,Nmr=.5, Nfill=y](n8)(37,52){}
\node[Nframe=y, Nh=1,Nw=1,Nmr=.5, Nfill=y](n9)(42,25){} 
\node[Nframe=y, Nh=1,Nw=1,Nmr=.5, Nfill=y](n10)(47,70){}
\node[Nframe=y, Nh=1,Nw=1,Nmr=.5, Nfill=y](n11)(52,65){} 
\node[Nframe=y, Nh=1,Nw=1,Nmr=.5, Nfill=y](n12)(57,75){}
\node[Nframe=y, Nh=1,Nw=1,Nmr=.5, Nfill=y](n13)(62,60){}
\node[Nframe=y, Nh=1,Nw=1,Nmr=.5, Nfill=y](n14)(67,75){}
\node[Nframe=y, Nh=1,Nw=1,Nmr=.5, Nfill=y](n15)(72,70){}
\node[Nframe=y, Nh=1,Nw=1,Nmr=.5, Nfill=y](n16)(77,85){}
\node[Nframe=y, Nh=1,Nw=1,Nmr=.5, Nfill=y](n17)(82,80){}
\node[Nframe=y, Nh=1,Nw=1,Nmr=.5, Nfill=y](n18)(87,95){}
\node[Nframe=y, Nh=1,Nw=1,Nmr=.5, Nfill=y](n19)(92,90){}
\node[Nframe=n, Nh=1,Nw=1,Nmr=.5](n20)(95,99){}

\drawedge[ELpos=50, ELside=r, curvedepth=0, AHnb=0, linegray=.5](n1,n2){}
\drawedge[ELpos=50, ELside=r, curvedepth=0, AHnb=0, linegray=.5](n2,n3){}
\drawedge[ELpos=50, ELside=r, curvedepth=0, AHnb=0, linegray=.5](n3,n4){}
\drawedge[ELpos=50, ELside=r, curvedepth=0, AHnb=0, linegray=.5](n4,n5){}
\drawedge[ELpos=50, ELside=r, curvedepth=0, AHnb=0, linegray=.5](n5,n6){}
\drawedge[ELpos=50, ELside=r, curvedepth=0, AHnb=0, linegray=.5](n6,n7){}
\drawedge[ELpos=50, ELside=r, curvedepth=0, AHnb=0, linegray=.5](n7,n8){}
\drawedge[ELpos=50, ELside=r, curvedepth=0, AHnb=0, linegray=.5](n8,n9){}
\drawedge[ELpos=50, ELside=r, curvedepth=0, AHnb=0, linegray=.5](n9,n10){}
\drawedge[ELpos=50, ELside=r, curvedepth=0, AHnb=0, linegray=.5](n10,n11){}
\drawedge[ELpos=50, ELside=r, curvedepth=0, AHnb=0, linegray=.5](n11,n12){}
\drawedge[ELpos=50, ELside=r, curvedepth=0, AHnb=0, linegray=.5](n12,n13){}
\drawedge[ELpos=50, ELside=r, curvedepth=0, AHnb=0, linegray=.5](n13,n14){}
\drawedge[ELpos=50, ELside=r, curvedepth=0, AHnb=0, linegray=.5](n14,n15){}
\drawedge[ELpos=50, ELside=r, curvedepth=0, AHnb=0, linegray=.5](n15,n16){}
\drawedge[ELpos=50, ELside=r, curvedepth=0, AHnb=0, linegray=.5](n16,n17){}
\drawedge[ELpos=50, ELside=r, curvedepth=0, AHnb=0, linegray=.5](n17,n18){}
\drawedge[ELpos=50, ELside=r, curvedepth=0, AHnb=0, linegray=.5](n18,n19){}
\drawedge[ELpos=50, ELside=r, curvedepth=0, AHnb=0, linegray=.5, dash={1}{1.5}](n19,n20){}








\end{picture}
}
     \subcaption{For the example of \figurename~\ref{fig:geometry2}. {\huge \strut}}\label{fig:ch2}
  \end{minipage}
  \hrule
  \caption{Convex hull interpretation of the value of a path.}\label{fig:ch-all}
\end{figure*}

A corollary of the geometric interpretation lemma is that the value
of a path can be obtained as the intersection of the vertical line at
point $T$ with the boundary of the convex hull of the region above the sequence of
utilities, namely ${\mathrm convexHull}(\{(t,y) \in \nat \times \real \mid y \geq u_t \})$.
This result is illustrated in \figurename~\ref{fig:ch-all}.

\subsubsection{Simple lassos are sufficient}

A \emph{lasso} is a path of the form $A C^{\omega}$
where $A$ and $C$ are finite paths (with $C$ a nonempty cycle),
where $A C^{\omega}$ is $A$ followed by infinite repetition of the cycle $C$. 
A lasso 
is \emph{simple} if all strict prefixes of the finite path $AC$ are acyclic.
In other words, simple lassos correspond to stationary plans. 

We show that there is always a simple lasso with optimal value. 
Our proof has four steps. 
Given a path $\rho$ that gives the utility sequence $u$, let $\nu$ 
be the slope of $f_u(t)$. 
Given a cycle $C$ in the path $\rho$, let $S_C$ be the sum of the weights in $C$
and let $M_C = \frac{S_C}{\abs{C}}$ be the average weight of the cycle edges.
The cycle $C$ is {\em good} if $M_C \geq \nu$, i.e., the average 
weight of the cycle is at least $\nu$, and {\em bad} otherwise.
\begin{itemize}
\item First, we show (in Lemma~\ref{lem:good-cycle}) that every path  
contains a good cycle.

\item Second, we show (in Lemma~\ref{lem:lasso-exists}) that 
if the first cycle in a path is good, then repeating the cycle cannot 
decrease the value of the path. 

\item Third, we show (in Lemma~\ref{lem:remove-bad-cycle}) that removing 
a bad cycle from a path cannot decrease the value of the path. 

\item Finally, we show (in Lemma~\ref{lem:simple-lasso-suffice}) that
given any path, using the above two operations of removal of bad cycles 
and repetition of good cycles, we obtain a simple lasso that does not decrease the value of the original path.

\end{itemize}
Thus we establish that simple lassos (or stationary plans) are sufficient for optimality.
To formalize the ideas we consider the notion of cycle decomposition.

\smallskip\noindent{\em Cycle decomposition.}
The \emph{cycle decomposition} of a path $\rho = e_0 e_1 \dots$ is an infinite sequence of simple 
cycles $C_1, C_2, \dots$ obtained as follows: push successively $e_0, e_1, \dots$ onto a stack, and 
whenever we push an edge that closes a (simple) cycle, we remove the cycle from
the stack and append it to the cycle decomposition. Note that the stack content
is always a prefix of a path of length at most~$\abs{V}$. 

\begin{lemma}\label{lem:good-cycle}
Let $T \in \nat$. Given a path $\rho$ that induces a sequence $u$ of utilities, 
let $\nu = \min_{0 \leq t_1 \leq T} \,\, \inf_{t_2 \geq T} \,\, \frac{u_{t_2} - u_{t_1}}{t_2 - t_1}$. 
Then, in the cycle decomposition of $\rho$ there exists a simple cycle $C$ with $M_C \geq \nu$. 
\end{lemma}

\begin{proof}
Towards contradiction, assume that all the (finitely many) cycles $C$ in the cycle decomposition 
of $\rho$ are such that $M_C < \nu$. Let $t_1$ be a left-minimizer of $\rho$. 
Since all cycles in $\rho$ have average weight smaller than $\nu$, we have: 
$$ \liminf_{t_2 \to \infty} \frac{u_{t_2} - u_{t_1}}{t_2 - t_1} < \nu $$
Since the infimum is bounded by the liminf, it follows that
$$ \min_{0 \leq t_1 \leq T} \, \inf_{t_2 \geq T} \frac{u_{t_2} - u_{t_1}}{t_2 - t_1} < \nu $$
which is in contradiction with the definition of $\nu$.
\end{proof}

We show that repeating a good cycle, and removing a bad cycle from a path
cannot decrease the value of the path.

\begin{lemma}\label{lem:lasso-exists}
Let $T \in \nat$. If the first cycle $C$ in the cycle decomposition of a path $\rho$ is good,
i.e.,  $M_C \geq \nu$ where $\nu = \min_{0 \leq t_1 \leq T} \,\, \inf_{t_2 \geq T} \,\, \frac{u_{t_2} - u_{t_1}}{t_2 - t_1}$,
then there exists a lasso $\rho'$ such that $\val(\rho',T) \geq \val(\rho,T)$.
\end{lemma}



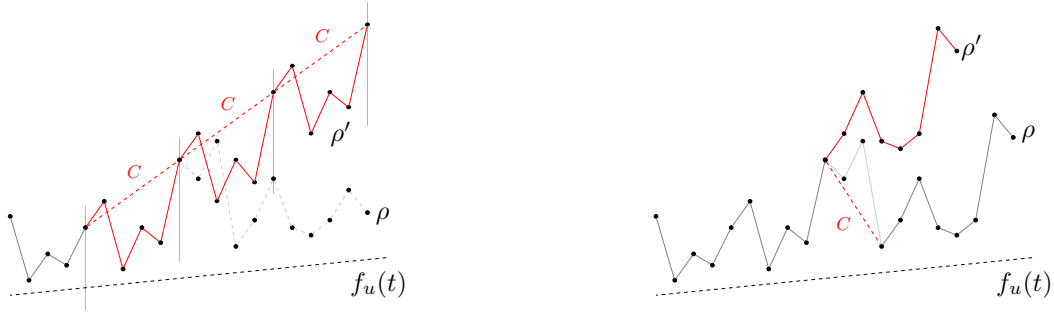
\begin{figure*}
  \hrule
  \begin{minipage}[b]{.48\linewidth}
     \centering 

\scalebox{0.5}{
\begin{picture}(105,85)(0,5)

\gasset{Nw=9,Nh=9,Nmr=4.5,rdist=1, loopdiam=6}   



\drawline[AHnb=0,arcradius=1, linegray=.7](22,5)(22,33)
\drawline[AHnb=0,arcradius=1, linegray=.7](47,18)(47,51)
\drawline[AHnb=0,arcradius=1, linegray=.7](72,36)(72,69)
\drawline[AHnb=0,arcradius=1, linegray=.7](97,54)(97,87)

\drawline[AHnb=0,arcradius=1, dash={1}0](2,9)(102,19)
\node[Nmarks=n, Nframe=n](n1)(100,12){\scalebox{2}{$f_u(t)$}}



\node[Nmarks=n, Nframe=n](n1)(101,30){\scalebox{2}{$\rho$}}
\node[Nmarks=n, Nframe=n](n1)(90,52){\scalebox{2}{$\rho'$}}

\node[Nmarks=n, Nframe=n](n1)(35,42){\scalebox{1.4}{\red{$C$}}}
\node[Nmarks=n, Nframe=n](n1)(60,60){\scalebox{1.4}{\red{$C$}}}
\node[Nmarks=n, Nframe=n](n1)(85,78){\scalebox{1.4}{\red{$C$}}}

\node[Nframe=y, Nh=1,Nw=1,Nmr=.5, Nfill=y](n1)(2,30){}
\node[Nframe=y, Nh=1,Nw=1,Nmr=.5, Nfill=y](n2)(7,13){}
\node[Nframe=y, Nh=1,Nw=1,Nmr=.5, Nfill=y](n3)(12,20){}
\node[Nframe=y, Nh=1,Nw=1,Nmr=.5, Nfill=y](n4)(17,17){}
\node[Nframe=y, Nh=1,Nw=1,Nmr=.5, Nfill=y](n5)(22,27){}
\node[Nframe=y, Nh=1,Nw=1,Nmr=.5, Nfill=y](n6)(27,34){}
\node[Nframe=y, Nh=1,Nw=1,Nmr=.5, Nfill=y](n7)(32,16){}
\node[Nframe=y, Nh=1,Nw=1,Nmr=.5, Nfill=y](n8)(37,27){}
\node[Nframe=y, Nh=1,Nw=1,Nmr=.5, Nfill=y](n9)(42,23){}
\node[Nframe=y, Nh=1,Nw=1,Nmr=.5, Nfill=y](n10)(47,45){}
\node[Nframe=y, Nh=1,Nw=1,Nmr=.5, Nfill=y](n11)(52,40){} 
\node[Nframe=y, Nh=1,Nw=1,Nmr=.5, Nfill=y](n12)(57,50){}
\node[Nframe=y, Nh=1,Nw=1,Nmr=.5, Nfill=y](n13)(62,22){} 
\node[Nframe=y, Nh=1,Nw=1,Nmr=.5, Nfill=y](n14)(67,29){}
\node[Nframe=y, Nh=1,Nw=1,Nmr=.5, Nfill=y](n15)(72,40){}
\node[Nframe=y, Nh=1,Nw=1,Nmr=.5, Nfill=y](n16)(77,27){}
\node[Nframe=y, Nh=1,Nw=1,Nmr=.5, Nfill=y](n17)(82,25){}
\node[Nframe=y, Nh=1,Nw=1,Nmr=.5, Nfill=y](n18)(87,29){}
\node[Nframe=y, Nh=1,Nw=1,Nmr=.5, Nfill=y](n19)(92,37){}
\node[Nframe=y, Nh=1,Nw=1,Nmr=.5, Nfill=y](n20)(97,31){}


\drawedge[ELpos=50, ELside=r, curvedepth=0, AHnb=0, linegray=.5](n1,n2){}
\drawedge[ELpos=50, ELside=r, curvedepth=0, AHnb=0, linegray=.5](n2,n3){}
\drawedge[ELpos=50, ELside=r, curvedepth=0, AHnb=0, linegray=.5](n3,n4){}
\drawedge[ELpos=50, ELside=r, curvedepth=0, AHnb=0, linegray=.5](n4,n5){}
\drawedge[ELpos=50, ELside=r, curvedepth=0, AHnb=0, linecolor=red](n5,n6){}
\drawedge[ELpos=50, ELside=r, curvedepth=0, AHnb=0, linecolor=red](n6,n7){}
\drawedge[ELpos=50, ELside=r, curvedepth=0, AHnb=0, linecolor=red](n7,n8){}
\drawedge[ELpos=50, ELside=r, curvedepth=0, AHnb=0, linecolor=red](n8,n9){}
\drawedge[ELpos=50, ELside=r, curvedepth=0, AHnb=0, linecolor=red](n9,n10){}
\gasset{dash={1}0}
\drawedge[ELpos=50, ELside=r, curvedepth=0, AHnb=0, linegray=.8](n10,n11){}
\drawedge[ELpos=50, ELside=r, curvedepth=0, AHnb=0, linegray=.8](n11,n12){}
\drawedge[ELpos=50, ELside=r, curvedepth=0, AHnb=0, linegray=.8](n12,n13){}
\drawedge[ELpos=50, ELside=r, curvedepth=0, AHnb=0, linegray=.8](n13,n14){}
\drawedge[ELpos=50, ELside=r, curvedepth=0, AHnb=0, linegray=.8](n14,n15){}
\drawedge[ELpos=50, ELside=r, curvedepth=0, AHnb=0, linegray=.8](n15,n16){}
\drawedge[ELpos=50, ELside=r, curvedepth=0, AHnb=0, linegray=.8](n16,n17){}
\drawedge[ELpos=50, ELside=r, curvedepth=0, AHnb=0, linegray=.8](n17,n18){}
\drawedge[ELpos=50, ELside=r, curvedepth=0, AHnb=0, linegray=.8](n18,n19){}
\drawedge[ELpos=50, ELside=r, curvedepth=0, AHnb=0, linegray=.8](n19,n20){}

\node[Nframe=y, Nh=1,Nw=1,Nmr=.5, Nfill=y](n10)(47,45){}
\node[Nframe=y, Nh=1,Nw=1,Nmr=.5, Nfill=y](n11)(52,52){}
\node[Nframe=y, Nh=1,Nw=1,Nmr=.5, Nfill=y](n12)(57,34){}
\node[Nframe=y, Nh=1,Nw=1,Nmr=.5, Nfill=y](n13)(62,45){}
\node[Nframe=y, Nh=1,Nw=1,Nmr=.5, Nfill=y](n14)(67,39){}
\node[Nframe=y, Nh=1,Nw=1,Nmr=.5, Nfill=y](n15)(72,63){}
\node[Nframe=y, Nh=1,Nw=1,Nmr=.5, Nfill=y](n16)(77,70){}
\node[Nframe=y, Nh=1,Nw=1,Nmr=.5, Nfill=y](n17)(82,52){}
\node[Nframe=y, Nh=1,Nw=1,Nmr=.5, Nfill=y](n18)(87,63){}
\node[Nframe=y, Nh=1,Nw=1,Nmr=.5, Nfill=y](n19)(92,59){}
\node[Nframe=y, Nh=1,Nw=1,Nmr=.5, Nfill=y](n20)(97,81){}

\drawedge[ELpos=50, ELside=r, curvedepth=0, AHnb=0, linecolor=red](n5,n10){}
\drawedge[ELpos=50, ELside=r, curvedepth=0, AHnb=0, linecolor=red](n10,n15){}
\drawedge[ELpos=50, ELside=r, curvedepth=0, AHnb=0, linecolor=red](n15,n20){}

\gasset{linegray=.5,dash={}0}
\gasset{linecolor=red}
\drawedge[ELpos=50, ELside=r, curvedepth=0, AHnb=0](n10,n11){}
\drawedge[ELpos=50, ELside=r, curvedepth=0, AHnb=0](n11,n12){}
\drawedge[ELpos=50, ELside=r, curvedepth=0, AHnb=0](n12,n13){}
\drawedge[ELpos=50, ELside=r, curvedepth=0, AHnb=0](n13,n14){}
\drawedge[ELpos=50, ELside=r, curvedepth=0, AHnb=0](n14,n15){}
\drawedge[ELpos=50, ELside=r, curvedepth=0, AHnb=0](n15,n16){}
\drawedge[ELpos=50, ELside=r, curvedepth=0, AHnb=0](n16,n17){}
\drawedge[ELpos=50, ELside=r, curvedepth=0, AHnb=0](n17,n18){}
\drawedge[ELpos=50, ELside=r, curvedepth=0, AHnb=0](n18,n19){}
\drawedge[ELpos=50, ELside=r, curvedepth=0, AHnb=0](n19,n20){}



\end{picture}

}
     \subcaption{Repeating a good cycle (Lemma~\ref{lem:lasso-exists}). {\huge \strut}}\label{fig:repeat-lasso}
  \end{minipage}%
  \hfill
  \begin{minipage}[b]{.48\linewidth}
     \centering 

\scalebox{0.5}{
\begin{picture}(105,85)(0,5)

\gasset{Nw=9,Nh=9,Nmr=4.5,rdist=1, loopdiam=6}   



\drawline[AHnb=0,arcradius=1, dash={1}0](2,9)(102,19)
\node[Nmarks=n, Nframe=n](n1)(100,12){\scalebox{2}{$f_u(t)$}}

\node[Nmarks=n, Nframe=n](n1)(101,52){\scalebox{2}{$\rho$}}
\node[Nmarks=n, Nframe=n](n1)(86,75){\scalebox{2}{$\rho'$}}

\node[Nmarks=n, Nframe=n](n1)(52,28){\scalebox{1.4}{\red{$C$}}}



\node[Nframe=y, Nh=1,Nw=1,Nmr=.5, Nfill=y](n1)(2,30){}
\node[Nframe=y, Nh=1,Nw=1,Nmr=.5, Nfill=y](n2)(7,13){}
\node[Nframe=y, Nh=1,Nw=1,Nmr=.5, Nfill=y](n3)(12,20){}
\node[Nframe=y, Nh=1,Nw=1,Nmr=.5, Nfill=y](n4)(17,17){}
\node[Nframe=y, Nh=1,Nw=1,Nmr=.5, Nfill=y](n5)(22,27){}
\node[Nframe=y, Nh=1,Nw=1,Nmr=.5, Nfill=y](n6)(27,34){}
\node[Nframe=y, Nh=1,Nw=1,Nmr=.5, Nfill=y](n7)(32,16){}
\node[Nframe=y, Nh=1,Nw=1,Nmr=.5, Nfill=y](n8)(37,27){}
\node[Nframe=y, Nh=1,Nw=1,Nmr=.5, Nfill=y](n9)(42,23){}
\node[Nframe=y, Nh=1,Nw=1,Nmr=.5, Nfill=y](n10)(47,45){}
\node[Nframe=y, Nh=1,Nw=1,Nmr=.5, Nfill=y](n11)(52,40){} 
\node[Nframe=y, Nh=1,Nw=1,Nmr=.5, Nfill=y](n12)(57,50){}
\node[Nframe=y, Nh=1,Nw=1,Nmr=.5, Nfill=y](n13)(62,22){} 
\node[Nframe=y, Nh=1,Nw=1,Nmr=.5, Nfill=y](n14)(67,29){}
\node[Nframe=y, Nh=1,Nw=1,Nmr=.5, Nfill=y](n15)(72,40){}
\node[Nframe=y, Nh=1,Nw=1,Nmr=.5, Nfill=y](n16)(77,27){}
\node[Nframe=y, Nh=1,Nw=1,Nmr=.5, Nfill=y](n17)(82,25){}
\node[Nframe=y, Nh=1,Nw=1,Nmr=.5, Nfill=y](n18)(87,29){}
\node[Nframe=y, Nh=1,Nw=1,Nmr=.5, Nfill=y](n19)(92,57){}
\node[Nframe=y, Nh=1,Nw=1,Nmr=.5, Nfill=y](n20)(97,51){}


\drawedge[ELpos=50, ELside=r, curvedepth=0, AHnb=0, linegray=.5](n1,n2){}
\drawedge[ELpos=50, ELside=r, curvedepth=0, AHnb=0, linegray=.5](n2,n3){}
\drawedge[ELpos=50, ELside=r, curvedepth=0, AHnb=0, linegray=.5](n3,n4){}
\drawedge[ELpos=50, ELside=r, curvedepth=0, AHnb=0, linegray=.5](n4,n5){}
\drawedge[ELpos=50, ELside=r, curvedepth=0, AHnb=0, linegray=.5](n5,n6){}
\drawedge[ELpos=50, ELside=r, curvedepth=0, AHnb=0, linegray=.5](n6,n7){}
\drawedge[ELpos=50, ELside=r, curvedepth=0, AHnb=0, linegray=.5](n7,n8){}
\drawedge[ELpos=50, ELside=r, curvedepth=0, AHnb=0, linegray=.5](n8,n9){}
\drawedge[ELpos=50, ELside=r, curvedepth=0, AHnb=0, linegray=.5](n9,n10){}
\drawedge[ELpos=50, ELside=r, curvedepth=0, AHnb=0, linegray=.8](n10,n11){}
\drawedge[ELpos=50, ELside=r, curvedepth=0, AHnb=0, linegray=.8](n11,n12){}
\drawedge[ELpos=50, ELside=r, curvedepth=0, AHnb=0, linegray=.8](n12,n13){}
\drawedge[ELpos=50, ELside=r, curvedepth=0, AHnb=0, linegray=.5](n13,n14){}
\drawedge[ELpos=50, ELside=r, curvedepth=0, AHnb=0, linegray=.5](n14,n15){}
\drawedge[ELpos=50, ELside=r, curvedepth=0, AHnb=0, linegray=.5](n15,n16){}
\drawedge[ELpos=50, ELside=r, curvedepth=0, AHnb=0, linegray=.5](n16,n17){}
\drawedge[ELpos=50, ELside=r, curvedepth=0, AHnb=0, linegray=.5](n17,n18){}
\drawedge[ELpos=50, ELside=r, curvedepth=0, AHnb=0, linegray=.5](n18,n19){}
\drawedge[ELpos=50, ELside=r, curvedepth=0, AHnb=0, linegray=.5](n19,n20){}

\drawedge[ELpos=50, ELside=r, curvedepth=0, AHnb=0, linecolor=red, dash={1}0](n10,n13){}

\node[Nframe=y, Nh=1,Nw=1,Nmr=.5, Nfill=y](n10)(47,45){}
\node[Nframe=y, Nh=1,Nw=1,Nmr=.5, Nfill=y](n11)(52,52){}
\node[Nframe=y, Nh=1,Nw=1,Nmr=.5, Nfill=y](n12)(57,63){}
\node[Nframe=y, Nh=1,Nw=1,Nmr=.5, Nfill=y](n13)(62,50){}
\node[Nframe=y, Nh=1,Nw=1,Nmr=.5, Nfill=y](n14)(67,48){}
\node[Nframe=y, Nh=1,Nw=1,Nmr=.5, Nfill=y](n15)(72,52){}
\node[Nframe=y, Nh=1,Nw=1,Nmr=.5, Nfill=y](n16)(77,80){}
\node[Nframe=y, Nh=1,Nw=1,Nmr=.5, Nfill=y](n17)(82,74){}

\gasset{linegray=.5}
\gasset{linecolor=red}
\drawedge[ELpos=50, ELside=r, curvedepth=0, AHnb=0](n10,n11){}
\drawedge[ELpos=50, ELside=r, curvedepth=0, AHnb=0](n11,n12){}
\drawedge[ELpos=50, ELside=r, curvedepth=0, AHnb=0](n12,n13){}
\drawedge[ELpos=50, ELside=r, curvedepth=0, AHnb=0](n13,n14){}
\drawedge[ELpos=50, ELside=r, curvedepth=0, AHnb=0](n14,n15){}
\drawedge[ELpos=50, ELside=r, curvedepth=0, AHnb=0](n15,n16){}
\drawedge[ELpos=50, ELside=r, curvedepth=0, AHnb=0](n16,n17){}


\end{picture}

}
     \subcaption{Removing a bad cycle (Lemma~\ref{lem:remove-bad-cycle}). {\huge \strut}}\label{fig:remove-lasso}
  \end{minipage}
  \hrule
  \caption{Constructing a lasso without decreasing the value (Lemma~\ref{lem:lasso-exists} and Lemma~\ref{lem:remove-bad-cycle}).}\label{fig:lasso-all}
\end{figure*}

\begin{proof}
Let $u$ be the sequence of utilities induced by $\rho$.
Since $C$ is the first cycle in $\rho$, there is a prefix of $\rho$ of the form $A C$ where $A$ is a finite path. 
Consider the lasso $\rho' = A C^{\omega}$ and its induced sequence of utilities $u'$. 

We show that the value of $\rho'$ is at least the value of $\rho$. 
By Lemma~\ref{lem:geometric-interpretation}, the optimal value of $u$ is $f_u(T)$,
and the sequence $u$ is above the line $f_u(t)$ (which has slope $\nu$),
i.e., $u(t) \geq f_u(t)$ for all $t \geq 0$. By Lemma~\ref{lem:lower-line}
it is sufficient to show that $u'$ is above the line $f_u(t)$ to establish that 
the optimal value of $u'$ is at least $f_u(T)$, that is $\val(\rho',T) \geq \val(\rho,T)$, 
and conclude the proof (the argument is illustrated in \figurename~\ref{fig:repeat-lasso}).

We show that $u'(t) \geq f_u(t)$ for all $t \geq 0$:
either $t \leq \abs{A} + \abs{C}$, and then $u'(t) = u(t) \geq f_u(t)$,
or $t > \abs{A} + \abs{C}$, and then let $k \in \nat$ such that $\abs{A} \leq t- k\cdot \abs{C} \leq \abs{A} + \abs{C}$,
and we have 
\begin{align*}
  u'(t)         & = u(t - k\cdot \abs{C}) + k \cdot S_C & (\rho' = A C^{\omega}) \\[+4pt]
                & \geq f_u(t - k\cdot \abs{C}) + k \cdot M_C \cdot \abs{C} & (u \text{ is above } f_u(t) \text{ and } S_C = M_C \cdot \abs{C}) \\[+4pt]
                & \geq f_u(t) - \nu \cdot k \cdot \abs{C} + k \cdot M_C \cdot \abs{C} & (f_u(t) \text{ is linear with slope } \nu)  \\[+4pt]
                & \geq f_u(t) + k \cdot \abs{C} \cdot (M_C - \nu) \\[+4pt]
                & \geq f_u(t). & (M_C \geq \nu) \\[+4pt]
\end{align*}
\end{proof}


\begin{lemma}\label{lem:remove-bad-cycle}
Let $T \in \nat$. 
If a path $\rho$ contains a bad cycle $C$,
that is such that $M_C < \nu$ where $\nu = \min_{0 \leq t_1 \leq T} \,\, \inf_{t_2 \geq T} \,\, \frac{u_{t_2} - u_{t_1}}{t_2 - t_1}$,
then removing $C$ from $\rho$ gives a path $\rho'$ such that $\val(\rho',T) \geq \val(\rho,T)$. 
\end{lemma}

\begin{proof}
Let $u, u'$ be the sequences of utilities induced by respectively $\rho$ and $\rho'$,
By the same argument as in the proof of Lemma~\ref{lem:lasso-exists} (using Lemma~\ref{lem:lower-line}
and Lemma~\ref{lem:geometric-interpretation}), it is sufficient to show that $u'$ is above the line $f_u(t)$.
Since $C$ is a cycle in $\rho$, there is a prefix of $\rho$ of the form $A C$ where $A$ is a finite path, and
for all $t \geq 0$ we have (the argument is illustrated in \figurename~\ref{fig:remove-lasso}):
either $t \leq \abs{A}$, then $u'(t) = u(t) \geq f_u(t)$, or $t > \abs{A}$, and then
\begin{align*}
  u'(t)         & = u(t + \abs{C}) - S_C & (C \text{ is removed from } \rho \text{ to get } \rho') \\[+4pt]
                & \geq f_u(t + \abs{C}) - M_C \cdot \abs{C} & (u \text{ is above } f_u(t) \text{ and } S_C = M_C \cdot \abs{C}) \\[+4pt]
                & \geq f_u(t) + \nu \cdot \abs{C} - M_C \cdot \abs{C} & (f_u(t) \text{ is linear with slope } \nu)  \\[+4pt]
                & \geq f_u(t) + \abs{C} \cdot (\nu - M_C) \\[+4pt]
                & \geq f_u(t). & (M_C < \nu) \\[+4pt]
\end{align*}
\end{proof}

Now we can show how to construct a simple lasso with value at least the value
of a given arbitrary path, and it follows that simple lassos are sufficient
for optimality.

\begin{lemma}\label{lem:simple-lasso-suffice}
Let $T \in \nat$. There exists a simple lasso $A C^{\omega}$ such that $\val(A C^{\omega},T)=\val(G,T)$.
\end{lemma}

\begin{proof}
Given an arbitrary path $\rho$, we construct a simple lasso with at least 
the same value as $\rho$. It follows that the optimal value is obtained
by stationary plans. The construction repeats the following steps:

\begin{enumerate}
\item Let $C$ be the first cycle in the cycle decomposition of $\rho$;

\item if $C$ is a bad cycle for the original path $\rho$, then we remove it to obtain
a new path $\rho'$. We continue the procedure with $\rho'$ (go to step $1.$);

\item otherwise $C$ is a good cycle for the original path $\rho$. 
Let $A$ be the prefix of $\rho$ until $C$ starts, and we construct the lasso $A C^{\omega}$.
\end{enumerate}

First, note that if the above procedure terminates, then the constructed
lasso has a value at least the value of the original path $\rho$ (by 
Lemma~\ref{lem:lasso-exists} and Lemma~\ref{lem:remove-bad-cycle}),
and it is a simple lasso by definition of the cycle decomposition.

Now we show that the procedure always terminates. By Lemma~\ref{lem:good-cycle}, 
there always exists a good cycle in the cycle decomposition of $\rho$,
and thus eventually a good cycle becomes the first cycle in the path
constructed by the above procedure, which then terminates.
\end{proof}

\noindent Theorem~\ref{thm:plans} follows from the above lemmas.

\begin{figure*}
  \hrule
  \begin{minipage}[b]{.45\linewidth}
     \centering 

\scalebox{0.5}{
\begin{picture}(135,95)(0,8)

\gasset{Nw=6,Nh=6,Nmr=3, rdist=1, loopdiam=5}

\drawline[AHnb=1,arcradius=1](2,45)(110,45)

\drawline[AHnb=0,arcradius=1](2,43)(2,47)
\drawline[AHnb=0,arcradius=1](72,43)(72,47)

\node[Nmarks=n, Nframe=n](n1)(2,40){\scalebox{2}{$0$}}
\node[Nmarks=n, Nframe=n](n1)(72,40){\scalebox{2}{$T$}}
\node[Nmarks=n, Nframe=n](n1)(110,49){\scalebox{2}{$t$}}

\drawline[AHnb=0,arcradius=1, dash={1}0 ](22,20)(92,55)      
\drawline[AHnb=0,arcradius=1, dash={1}0 ](12,30)(92,50)      

\node[Nmarks=n, Nframe=n](n1)(13,62){\scalebox{2}{$\rho'$}}
\node[Nmarks=n, Nframe=n](n1)(7,49){\scalebox{2}{$\rho$}}


\put(55,67){\makebox(0,0)[l]{\scalebox{2}{$\rho$ has greater sum}}}
\put(55,60){\makebox(0,0)[l]{\scalebox{2}{of weights}}}
\put(50,30){\makebox(0,0)[l]{\scalebox{2}{$\rho$ is less constraining on}}}
\put(50,23){\makebox(0,0)[l]{\scalebox{2}{the slope of the line}}}
\put(50,16){\makebox(0,0)[l]{\scalebox{2}{$f(t) = M \cdot (t-T)$}}}


\node[Nframe=y, Nh=1,Nw=1,Nmr=.5, Nfill=y, fillgray=.5 ,linegray=.5](n1)(2,45){}
\node[Nframe=y, Nh=1,Nw=1,Nmr=.5, Nfill=y, fillgray=.5 ,linegray=.5](n2)(7,62){}
\node[Nframe=y, Nh=1,Nw=1,Nmr=.5, Nfill=y, fillgray=.5 ,linegray=.5](n3)(12,55){}
\node[Nframe=y, Nh=1,Nw=1,Nmr=.5, Nfill=y, fillgray=.5 ,linegray=.5](n4)(17,58){}
\node[Nframe=y, Nh=1,Nw=1,Nmr=.5, Nfill=y, fillgray=.5 ,linegray=.5](n5)(22,48){}
\node[Nframe=y, Nh=1,Nw=1,Nmr=.5, Nfill=y, fillgray=.5 ,linegray=.5](n6)(27,50){}
\node[Nframe=y, Nh=1,Nw=1,Nmr=.5, Nfill=y, fillgray=.5 ,linegray=.5](n7)(32,57){}
\node[Nframe=y, Nh=1,Nw=1,Nmr=.5, Nfill=y, fillgray=.5 ,linegray=.5](n8)(37,50){}
\node[Nframe=y, Nh=1,Nw=1,Nmr=.5, Nfill=y, fillgray=.5 ,linegray=.5](n9)(42,30){}
\node[Nframe=y, Nh=1,Nw=1,Nmr=.5, Nfill=y, fillgray=.5 ,linegray=.5](n10)(47,42){}
\node[Nframe=y, Nh=1,Nw=1,Nmr=.5, Nfill=y](n11)(52,50){} 

\drawedge[ELpos=50, ELside=r, curvedepth=0, AHnb=0, linegray=.5](n1,n2){}
\drawedge[ELpos=50, ELside=r, curvedepth=0, AHnb=0, linegray=.5](n2,n3){}
\drawedge[ELpos=50, ELside=r, curvedepth=0, AHnb=0, linegray=.5](n3,n4){}
\drawedge[ELpos=50, ELside=r, curvedepth=0, AHnb=0, linegray=.5](n4,n5){}
\drawedge[ELpos=50, ELside=r, curvedepth=0, AHnb=0, linegray=.5](n5,n6){}
\drawedge[ELpos=50, ELside=r, curvedepth=0, AHnb=0, linegray=.5](n6,n7){}
\drawedge[ELpos=50, ELside=r, curvedepth=0, AHnb=0, linegray=.5](n7,n8){}
\drawedge[ELpos=50, ELside=r, curvedepth=0, AHnb=0, linegray=.5](n8,n9){}
\drawedge[ELpos=50, ELside=r, curvedepth=0, AHnb=0, linegray=.5](n9,n10){}
\drawedge[ELpos=50, ELside=r, curvedepth=0, AHnb=0, linegray=.5](n10,n11){}



\node[Nframe=y, Nh=1,Nw=1,Nmr=.5, Nfill=y, fillgray=.5 ,linegray=.5](n1)(2,45){}
\node[Nframe=y, Nh=1,Nw=1,Nmr=.5, Nfill=y, fillgray=.5 ,linegray=.5](n2)(7,57){}
\node[Nframe=y, Nh=1,Nw=1,Nmr=.5, Nfill=y, fillgray=.5 ,linegray=.5](n3)(12,48){}
\node[Nframe=y, Nh=1,Nw=1,Nmr=.5, Nfill=y, fillgray=.5 ,linegray=.5](n4)(17,52){}
\node[Nframe=y, Nh=1,Nw=1,Nmr=.5, Nfill=y, fillgray=.5 ,linegray=.5](n5)(22,40){}
\node[Nframe=y, Nh=1,Nw=1,Nmr=.5, Nfill=y, fillgray=.5 ,linegray=.5](n6)(27,33.75){}
\node[Nframe=y, Nh=1,Nw=1,Nmr=.5, Nfill=y, fillgray=.5 ,linegray=.5](n7)(32,42){}
\node[Nframe=y, Nh=1,Nw=1,Nmr=.5, Nfill=y, fillgray=.5 ,linegray=.5](n8)(37,40){}
\node[Nframe=y, Nh=1,Nw=1,Nmr=.5, Nfill=y, fillgray=.5 ,linegray=.5](n9)(42,52){}
\node[Nframe=y, Nh=1,Nw=1,Nmr=.5, Nfill=y, fillgray=.5 ,linegray=.5](n10)(47,57){}
\node[Nframe=y, Nh=1,Nw=1,Nmr=.5, Nfill=y](n11)(52,55){} 



\gasset{linecolor=red}
\drawedge[ELpos=50, ELside=r, curvedepth=0, AHnb=0](n1,n2){}
\drawedge[ELpos=50, ELside=r, curvedepth=0, AHnb=0](n2,n3){}
\drawedge[ELpos=50, ELside=r, curvedepth=0, AHnb=0](n3,n4){}
\drawedge[ELpos=50, ELside=r, curvedepth=0, AHnb=0](n4,n5){}
\drawedge[ELpos=50, ELside=r, curvedepth=0, AHnb=0](n5,n6){}
\drawedge[ELpos=50, ELside=r, curvedepth=0, AHnb=0](n6,n7){}
\drawedge[ELpos=50, ELside=r, curvedepth=0, AHnb=0](n7,n8){}
\drawedge[ELpos=50, ELside=r, curvedepth=0, AHnb=0](n8,n9){}
\drawedge[ELpos=50, ELside=r, curvedepth=0, AHnb=0](n9,n10){}
\drawedge[ELpos=50, ELside=r, curvedepth=0, AHnb=0](n10,n11){}






\end{picture}
}
     \subcaption{The path length is smaller than $T$. {\huge \strut}}\label{fig:preferred}
  \end{minipage}%
  \hfill
  \begin{minipage}[b]{.54\linewidth}
     \centering 

\scalebox{0.5}{
\begin{picture}(165,95)(0,0)

\gasset{Nw=6,Nh=6,Nmr=3, rdist=1, loopdiam=5}

\drawline[AHnb=1,arcradius=1](2,37)(110,37)

\drawline[AHnb=0,arcradius=1](2,35)(2,39)
\drawline[AHnb=0,arcradius=1](72,35)(72,39)

\node[Nmarks=n, Nframe=n](n1)(2,32){\scalebox{2}{$0$}}
\node[Nmarks=n, Nframe=n](n1)(72,32){\scalebox{2}{$T$}}
\node[Nmarks=n, Nframe=n](n1)(110,41){\scalebox{2}{$t$}}

\drawline[AHnb=0,arcradius=1, dash={1}0](12,22)(72,37)      
\drawline[AHnb=0,arcradius=1, dash={1}0](12,29.5)(72,37)      

\drawline[AHnb=0,arcradius=1, dash={1}0](72,37)(122,87)      
\drawline[AHnb=0,arcradius=1, dash={1}0](72,37)(122,62)      

\node[Nmarks=n, Nframe=n](n1)(13,54){\scalebox{2}{$\rho'$}}
\node[Nmarks=n, Nframe=n](n1)(7,41){\scalebox{2}{$\rho$}}


\put(50,79){\makebox(0,0)[l]{\scalebox{2}{$\rho$ has greater sum}}}
\put(50,72){\makebox(0,0)[l]{\scalebox{2}{of weights}}}
\put(30,22){\makebox(0,0)[l]{\scalebox{2}{$\rho$ is less constraining on}}}
\put(30,15){\makebox(0,0)[l]{\scalebox{2}{the slope of the line}}}
\put(30,8){\makebox(0,0)[l]{\scalebox{2}{$f(t) = M \cdot (t-T)$}}}


\put(125,87){\makebox(0,0)[l]{\scalebox{1.5}{slope $M = 1$}}}
\put(125,62){\makebox(0,0)[l]{\scalebox{1.5}{slope $M = \frac{1}{2}$}}}

\put(105,15){\makebox(0,0)[l]{\scalebox{2}{
$\begin{array}{rl}
 \varphi_{\rho} \equiv & \!\!\!\!\frac{1}{8} \leq M \leq 1 \\[+4pt]
 \varphi_{\rho'} \equiv & \!\!\!\!\frac{1}{4} \leq M \leq \frac{1}{2}
 \end{array}
$
}}}


\node[Nframe=y, Nh=1,Nw=1,Nmr=.5, Nfill=y, fillgray=.5 ,linegray=.5](n1)(2,37){}
\node[Nframe=y, Nh=1,Nw=1,Nmr=.5, Nfill=y, fillgray=.5 ,linegray=.5](n2)(7,54){}
\node[Nframe=y, Nh=1,Nw=1,Nmr=.5, Nfill=y, fillgray=.5 ,linegray=.5](n3)(12,47){}
\node[Nframe=y, Nh=1,Nw=1,Nmr=.5, Nfill=y, fillgray=.5 ,linegray=.5](n4)(17,50){}
\node[Nframe=y, Nh=1,Nw=1,Nmr=.5, Nfill=y, fillgray=.5 ,linegray=.5](n5)(22,40){}
\node[Nframe=y, Nh=1,Nw=1,Nmr=.5, Nfill=y, fillgray=.5 ,linegray=.5](n6)(27,42){}
\node[Nframe=y, Nh=1,Nw=1,Nmr=.5, Nfill=y, fillgray=.5 ,linegray=.5](n7)(32,39){}
\node[Nframe=y, Nh=1,Nw=1,Nmr=.5, Nfill=y, fillgray=.5 ,linegray=.5](n8)(37,28.25){}
\node[Nframe=y, Nh=1,Nw=1,Nmr=.5, Nfill=y, fillgray=.5 ,linegray=.5](n9)(42,42){}
\node[Nframe=y, Nh=1,Nw=1,Nmr=.5, Nfill=y, fillgray=.5 ,linegray=.5](n10)(47,44){}
\node[Nframe=y, Nh=1,Nw=1,Nmr=.5, Nfill=y, fillgray=.5 ,linegray=.5](n11)(52,52){}
\node[Nframe=y, Nh=1,Nw=1,Nmr=.5, Nfill=y, fillgray=.5 ,linegray=.5](n12)(57,47){}
\node[Nframe=y, Nh=1,Nw=1,Nmr=.5, Nfill=y, fillgray=.5 ,linegray=.5](n13)(62,39){} 
\node[Nframe=y, Nh=1,Nw=1,Nmr=.5, Nfill=y, fillgray=.5 ,linegray=.5](n14)(67,49){}
\node[Nframe=y, Nh=1,Nw=1,Nmr=.5, Nfill=y, fillgray=.5 ,linegray=.5](n15)(72,42){}
\node[Nframe=y, Nh=1,Nw=1,Nmr=.5, Nfill=y, fillgray=.5 ,linegray=.5](n16)(77,46){}
\node[Nframe=y, Nh=1,Nw=1,Nmr=.5, Nfill=y, fillgray=.5 ,linegray=.5](n17)(82,42){}
\node[Nframe=y, Nh=1,Nw=1,Nmr=.5, Nfill=y, fillgray=.5 ,linegray=.5](n18)(87,48){}
\node[Nframe=y, Nh=1,Nw=1,Nmr=.5, Nfill=y](n19)(92,50){}

\drawedge[ELpos=50, ELside=r, curvedepth=0, AHnb=0, linegray=.5](n1,n2){}
\drawedge[ELpos=50, ELside=r, curvedepth=0, AHnb=0, linegray=.5](n2,n3){}
\drawedge[ELpos=50, ELside=r, curvedepth=0, AHnb=0, linegray=.5](n3,n4){}
\drawedge[ELpos=50, ELside=r, curvedepth=0, AHnb=0, linegray=.5](n4,n5){}
\drawedge[ELpos=50, ELside=r, curvedepth=0, AHnb=0, linegray=.5](n5,n6){}
\drawedge[ELpos=50, ELside=r, curvedepth=0, AHnb=0, linegray=.5](n6,n7){}
\drawedge[ELpos=50, ELside=r, curvedepth=0, AHnb=0, linegray=.5](n7,n8){}
\drawedge[ELpos=50, ELside=r, curvedepth=0, AHnb=0, linegray=.5](n8,n9){}
\drawedge[ELpos=50, ELside=r, curvedepth=0, AHnb=0, linegray=.5](n9,n10){}
\drawedge[ELpos=50, ELside=r, curvedepth=0, AHnb=0, linegray=.5](n10,n11){}
\drawedge[ELpos=50, ELside=r, curvedepth=0, AHnb=0, linegray=.5](n11,n12){}
\drawedge[ELpos=50, ELside=r, curvedepth=0, AHnb=0, linegray=.5](n12,n13){}
\drawedge[ELpos=50, ELside=r, curvedepth=0, AHnb=0, linegray=.5](n13,n14){}
\drawedge[ELpos=50, ELside=r, curvedepth=0, AHnb=0, linegray=.5](n14,n15){}
\drawedge[ELpos=50, ELside=r, curvedepth=0, AHnb=0, linegray=.5](n15,n16){}
\drawedge[ELpos=50, ELside=r, curvedepth=0, AHnb=0, linegray=.5](n16,n17){}
\drawedge[ELpos=50, ELside=r, curvedepth=0, AHnb=0, linegray=.5](n17,n18){}
\drawedge[ELpos=50, ELside=r, curvedepth=0, AHnb=0, linegray=.5](n18,n19){}


\node[Nframe=y, Nh=1,Nw=1,Nmr=.5, Nfill=y, fillgray=.5 ,linegray=.5](n1)(2,37){}
\node[Nframe=y, Nh=1,Nw=1,Nmr=.5, Nfill=y, fillgray=.5 ,linegray=.5](n2)(7,49){}
\node[Nframe=y, Nh=1,Nw=1,Nmr=.5, Nfill=y, fillgray=.5 ,linegray=.5](n3)(12,40){}
\node[Nframe=y, Nh=1,Nw=1,Nmr=.5, Nfill=y, fillgray=.5 ,linegray=.5](n4)(17,44){}
\node[Nframe=y, Nh=1,Nw=1,Nmr=.5, Nfill=y, fillgray=.5 ,linegray=.5](n5)(22,30.75){}
\node[Nframe=y, Nh=1,Nw=1,Nmr=.5, Nfill=y, fillgray=.5 ,linegray=.5](n6)(27,35){}
\node[Nframe=y, Nh=1,Nw=1,Nmr=.5, Nfill=y, fillgray=.5 ,linegray=.5](n7)(32,47){}
\node[Nframe=y, Nh=1,Nw=1,Nmr=.5, Nfill=y, fillgray=.5 ,linegray=.5](n8)(37,44){}
\node[Nframe=y, Nh=1,Nw=1,Nmr=.5, Nfill=y, fillgray=.5 ,linegray=.5](n9)(42,55){}
\node[Nframe=y, Nh=1,Nw=1,Nmr=.5, Nfill=y, fillgray=.5 ,linegray=.5](n10)(47,49){}
\node[Nframe=y, Nh=1,Nw=1,Nmr=.5, Nfill=y, fillgray=.5 ,linegray=.5](n11)(52,54){}
\node[Nframe=y, Nh=1,Nw=1,Nmr=.5, Nfill=y, fillgray=.5 ,linegray=.5](n12)(57,54){}
\node[Nframe=y, Nh=1,Nw=1,Nmr=.5, Nfill=y, fillgray=.5 ,linegray=.5](n13)(62,47){} 
\node[Nframe=y, Nh=1,Nw=1,Nmr=.5, Nfill=y, fillgray=.5 ,linegray=.5](n14)(67,55){}
\node[Nframe=y, Nh=1,Nw=1,Nmr=.5, Nfill=y, fillgray=.5 ,linegray=.5](n15)(72,46){}
\node[Nframe=y, Nh=1,Nw=1,Nmr=.5, Nfill=y, fillgray=.5 ,linegray=.5](n16)(77,53){}
\node[Nframe=y, Nh=1,Nw=1,Nmr=.5, Nfill=y, fillgray=.5 ,linegray=.5](n17)(82,57){}
\node[Nframe=y, Nh=1,Nw=1,Nmr=.5, Nfill=y, fillgray=.5 ,linegray=.5](n18)(87,52){}
\node[Nframe=y, Nh=1,Nw=1,Nmr=.5, Nfill=y](n19)(92,72){}




\gasset{linecolor=red}
\drawedge[ELpos=50, ELside=r, curvedepth=0, AHnb=0](n1,n2){}
\drawedge[ELpos=50, ELside=r, curvedepth=0, AHnb=0](n2,n3){}
\drawedge[ELpos=50, ELside=r, curvedepth=0, AHnb=0](n3,n4){}
\drawedge[ELpos=50, ELside=r, curvedepth=0, AHnb=0](n4,n5){}
\drawedge[ELpos=50, ELside=r, curvedepth=0, AHnb=0](n5,n6){}
\drawedge[ELpos=50, ELside=r, curvedepth=0, AHnb=0](n6,n7){}
\drawedge[ELpos=50, ELside=r, curvedepth=0, AHnb=0](n7,n8){}
\drawedge[ELpos=50, ELside=r, curvedepth=0, AHnb=0](n8,n9){}
\drawedge[ELpos=50, ELside=r, curvedepth=0, AHnb=0](n9,n10){}
\drawedge[ELpos=50, ELside=r, curvedepth=0, AHnb=0](n10,n11){}
\drawedge[ELpos=50, ELside=r, curvedepth=0, AHnb=0](n11,n12){}
\drawedge[ELpos=50, ELside=r, curvedepth=0, AHnb=0](n12,n13){}
\drawedge[ELpos=50, ELside=r, curvedepth=0, AHnb=0](n13,n14){}
\drawedge[ELpos=50, ELside=r, curvedepth=0, AHnb=0](n14,n15){}
\drawedge[ELpos=50, ELside=r, curvedepth=0, AHnb=0](n15,n16){}
\drawedge[ELpos=50, ELside=r, curvedepth=0, AHnb=0](n16,n17){}
\drawedge[ELpos=50, ELside=r, curvedepth=0, AHnb=0](n17,n18){}
\drawedge[ELpos=50, ELside=r, curvedepth=0, AHnb=0](n18,n19){}






\end{picture}
}
     \subcaption{The path length is greater than $T$. {\huge \strut}}\label{fig:preferred2}
  \end{minipage}
  \hrule
  \caption{The path $\rho$ is preferred to $\rho'$.}\label{fig:preferredall}
\end{figure*}
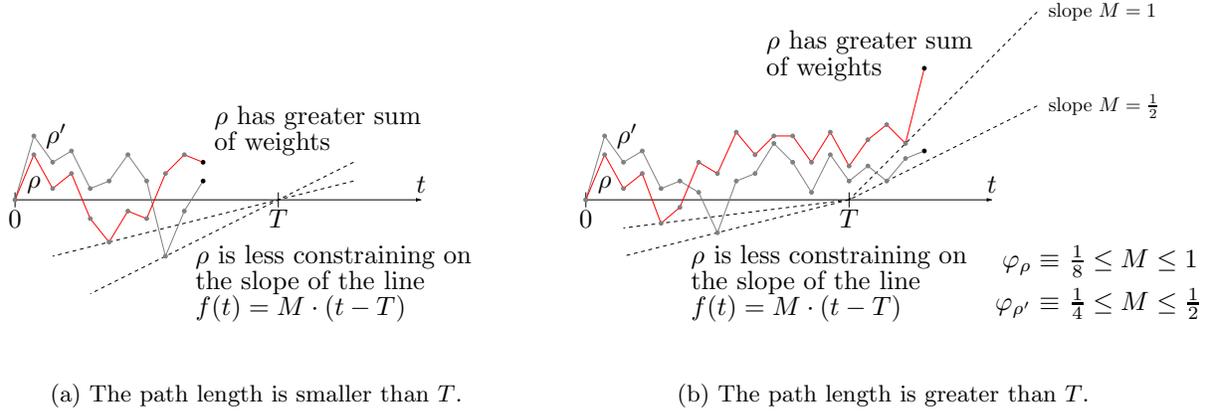

\subsection{Theorem~\ref{thm:algo}: Algorithm and Complexity Analysis}
In this section we present our algorithm and then the complexity analysis.

\subsubsection{Algorithm}
The key challenges to obtain an algorithm are as follows.
First, while for the fixed-horizon problem backward induction or powering of transition matrix leads
to an algorithm, for expected time horizon with an adversary, there is no a-priori bound on 
the number of steps, and hence the backward induction approach is not applicable. 
Second, stationary optimal plans suffice, and as shown in Theorem~\ref{thm:hard} computing 
optimal stationary plans for the fixed horizon problem is NP-hard.
We present an algorithm that iteratively constructs the {\em most promising} candidate paths
according to a partial order of the paths, and the key
is to define the partial order.

It follows from the geometric interpretation lemmas (Lemma~\ref{lem:lower-line} 
and Lemma~\ref{lem:geometric-interpretation}) that the value of a path is at 
least~$0$ if its sequence of utilities  is above some line that contains the 
point $(T,0)$.

\begin{lemma}\label{lem:exists-slope}
The value of a sequence $u$ of utilities is at least $0$ if and only if
there exists a slope $M \in \real$ such that $u_t \geq M \cdot (t-T)$ for all $t \geq 0$.
\end{lemma}

\begin{proof}
If the value of $u$ is at least $0$, then $f_u(T) \geq 0$ and
by Lemma~\ref{lem:geometric-interpretation} we have $u_t \geq f_u(t)$ for all $t \geq 0$. 
Then $u_t \geq f_u(t) - f_u(T)$ (which is a linear function of $t$) and we can take 
for $M$ the value of the coefficient of $t$ in the expression $f_u(t) - f_u(T)$.

To prove the other direction, consider the line of equation $f(t) = M \cdot (t-T)$,
and by Lemma~\ref{lem:lower-line}, the value of the sequence $u$ is at least $f(T) = 0$.
\end{proof}

\begin{algorithm}[t]
\caption{$\BestPaths(t_0,v_0,u_0,\psi_0)$}
\label{alg:best-path}
{
 \AlgData{$t_0 \in \nat$ is an initial time point, $v_0$ is an initial vertex,
 $u_0$ is the initial sum of weights, and $\psi_0$ is the initial constraint on the slope parameter $M$.}
 \AlgResult{The table of $\succeq$-maximal values of paths from $v_0$ with initial values $t_0,u_0,\psi_0$.}
\begin{flushleft}
 \Begin{
        \hfill {\tt /* initialization */} \;
	\nl $D[t_0,v_0] \gets \{ \tuple{u_0,\psi_0} \}$ \label{alg:init-begin} \; 
        \nl \For{$v \in V \setminus \{v_0\}$\negskip} 
        {
           \nl $D[t_0,v] \gets \emptyset$ \label{alg:init-end} \; 
        }
        \hfill {\tt /* iterations */} \;
        \nl \For{$i = 1, \dots, \abs{V}$\negskip \label{alg:best-path-first-for}} 
        {
          \nl \For{$v \in V$\negskip \label{alg:best-path-second-for}}
          {
             \nl $D[t_0 + i,v] \gets \emptyset$ \;
             \nl \For{{\footnotesize $v_1 \in V$ and $\tuple{u_1,\psi_1} \in D[t_0+i-1,v_1]$\!\!\negskip} \label{alg:best-path-for-D}}
             {
                \nl \If{$(v_1,v) \in E$\negskip}
                {
                   \nl $u \gets u_1 + w(v_1,v)$  \;
                   \nl $t \gets t_0 + i - 1$  \;
                   \nl $\psi \gets \psi_1 \land (r \geq M \cdot (t-T))$ \;
                   \nl {\small $D[t_0 + i,v] \gets D[t_0 + i,v] \cup \{ \tuple{u,\psi} \}$}\label{alg:best-path-add-D}\;
                }
             }
             \nl $D[t_0+i,v] \gets \ceil{D[t_0+i,v]}$ \label{alg:best-path-max-D}\;
          }
        }
	\nl \KwRet{$D$}\; 
 }
\end{flushleft}
}
\end{algorithm}

The expression $u_t - M \cdot (t-T)$ that appears in the condition of Lemma~\ref{lem:exists-slope}
can be obtained by subtracting $M$ to each weight of the graph, and shifting the
sum of the weights by the constant $T \cdot M$. Since $M$ is unknown, we can
define the following symbolic constraint on $M$ (associated with a path $\rho$) that ensures, if it is satisfiable,
that the sequence of utilities of $\rho = e_0 e_1 \dots e_k$ is above the line 
of equation $f(t) = M \cdot(t-T)$ :
$$ \varphi_{\rho} \equiv \bigwedge_{0 \leq i \leq k} (u_i \geq M \cdot (i-T)) $$

Note that $k = \abs{\rho} - 1$, and the constraint $\varphi_{\rho}$
represents an interval (possibly empty, possibly unbounded) of values for $M$.
Intuitively, a finite path is more promising (thus preferred) in order to be prolonged to 
an infinite path with value at least $0$ if the total sum of weights is large
and the constraint $\varphi_{\rho}$ is weak (see \figurename~\ref{fig:preferred} and \figurename~\ref{fig:preferred2}). 
To each finite path~$\rho$, we associate a pair $\tuple{u, \psi}$ consisting
of the sum $u$ of the weights in~$\rho$, and the constraint $\psi = \varphi_{\rho}$.

Given two pairs $\tuple{u, \psi}$, $\tuple{u', \psi'}$ (associated with 
paths $\rho$ and $\rho'$ respectively), 
we write $\tuple{u, \psi} \succeq \tuple{u', \psi'}$ 
if $u \geq u'$ and $\psi'$ implies $\psi$, and we say that 
$\rho$ is \emph{preferred} to $\rho'$ (this is a partial order).
Given a set $S$ of such pairs, 
denote by $\ceil{S} = \{z_1  \in S \mid \forall z_2 \in S:
z_2 \succeq z_1 \to z_1 \succeq z_2 \}$ the set of $\succeq$-maximal elements of $S$.
Note that the elements of $\ceil{S}$ are pairwise $\succeq$-incomparable.

Intuitively, if $\rho$ and $\rho'$ end in the same vertex, and $\rho$ is \emph{preferred} to $\rho'$,
then it is easier to extend $\rho$ than $\rho'$ to obtain an (infinite) path with expected value at least $0$.
Formally, for all infinite paths $\pi$ with $\starts(\pi) = \ends(\rho) = \ends(\rho')$ 
we have $\val(\rho \cdot \pi,T) \geq \val(\rho' \cdot \pi,T)$. We use this result
in the following form.

\begin{lemma}\label{lem:constraint-prefered} 
Let $\rho_1$, $\rho_{\!A}^{}$ be two paths of the same length with the same end state,
i.e., $\ends(\rho_1) = \ends(\rho_{\!A}^{})$.
If $\rho_1$ is preferred to $\rho_{\!A}^{}$, 
then for all paths $\rho_{C}^{}$ with $\starts(\rho_{C}^{}) = \ends(\rho_{\!A}^{})$, 
the path $\rho_1 \cdot \rho_{C}^{}$ is preferred to the path $\rho_{\!A}^{} \cdot \rho_{C}^{}$.
\end{lemma}

\begin{proof}[Proof sketch]
Let $\rho_{1C} = \rho_1 \cdot \rho_{C}^{}$ an $\rho_{AC} = \rho_{\!A}^{} \cdot \rho_{C}^{}$.
Denote by $u_1$, $u_A$, $u_{1C}$, and $u_{AC}$
the sum of the weights of the paths $\rho_1$, $\rho_{\!A}^{}$, $\rho_1 \cdot \rho_{C}^{}$,
and $\rho_{\!A}^{} \cdot \rho_{C}^{}$ respectively.

Since $u_1 \geq u_A$ and $\varphi_{\rho_{\!A}^{}} \to \varphi_{\rho_1}$, it is easy
to see that $u_{1C} \geq u_{AC}$, and that for every length $\abs{\rho_1} \leq k \leq \abs{\rho_1} + \abs{\rho_{C}^{}}$,
the sum of the weights of the prefix of length $k$ of $\rho_1 \cdot \rho_{C}^{}$
at least as large as the sum of the weights of the prefix of length $k$ of $\rho_{\!A}^{} \cdot \rho_{C}^{}$.
It follows that $\varphi_{\rho_{AC}} \to \varphi_{\rho_{1C}}$ as well,
hence $\rho_1 \cdot \rho_{C}^{}$ is preferred to $\rho_{\!A}^{} \cdot \rho_{C}^{}$.
\end{proof}

\begin{figure*}[!tb]
  \begin{center}
    \hrule

\scalebox{0.5}{
\begin{picture}(220,54)(0,40)

\gasset{Nw=6,Nh=6,Nmr=3, rdist=1, loopdiam=5}

\drawline[AHnb=1,arcradius=1](2,55)(110,55)

\drawline[AHnb=0,arcradius=1](2,53)(2,57)
\drawline[AHnb=0,arcradius=1](52,53)(52,57)
\drawline[AHnb=0,arcradius=1](92,53)(92,57)

\node[Nmarks=n, Nframe=n](n1)(2,50){\scalebox{2}{$0$}}
\node[Nmarks=n, Nframe=n](n1)(52,50){\scalebox{2}{$t_0$}}
\node[Nmarks=n, Nframe=n](n1)(92,50){\scalebox{2}{$t_0 + i$}}

\node[Nmarks=n, Nframe=n](n1)(13,72){\scalebox{2}{$\rho_{\sharp}$}}
\node[Nmarks=n, Nframe=n](n1)(48,67){\scalebox{2}{$u_0$}}
\node[Nmarks=n, Nframe=n](n1)(55,60){\scalebox{2}{$v_0$}}

\node[Nmarks=n, Nframe=n](n1)(69,77){\scalebox{2}{$\rho_1$}}
\node[Nmarks=n, Nframe=n](n1)(89,89){\scalebox{2}{$u$}}
\node[Nmarks=n, Nframe=n](n1)(96,80){\scalebox{2}{$v_1$}}
\node[Nmarks=n, Nframe=n](n1)(105,64){\scalebox{2}{$\psi \equiv \varphi_{\rho_{\sharp} \cdot \rho_1}$}}

\put(130,75){\makebox(0,0)[l]{\scalebox{2}{$\tuple{u,\psi} \in D[t_0+i,v_1]$}}}
\put(130,65){\makebox(0,0)[l]{\scalebox{2}{where $D = \BestPaths(t_0,v_0,u_0,\psi_0)$ }}}

\node[Nframe=y, Nh=1,Nw=1,Nmr=.5, Nfill=y, fillgray=.5 ,linegray=.5](n1)(2,55){}
\node[Nframe=y, Nh=1,Nw=1,Nmr=.5, Nfill=y, fillgray=.5 ,linegray=.5](n2)(7,72){}
\node[Nframe=y, Nh=1,Nw=1,Nmr=.5, Nfill=y, fillgray=.5 ,linegray=.5](n3)(12,65){}
\node[Nframe=y, Nh=1,Nw=1,Nmr=.5, Nfill=y, fillgray=.5 ,linegray=.5](n4)(17,68){}
\node[Nframe=y, Nh=1,Nw=1,Nmr=.5, Nfill=y, fillgray=.5 ,linegray=.5](n5)(22,58){}
\node[Nframe=y, Nh=1,Nw=1,Nmr=.5, Nfill=y, fillgray=.5 ,linegray=.5](n6)(27,60){}
\node[Nframe=y, Nh=1,Nw=1,Nmr=.5, Nfill=y, fillgray=.5 ,linegray=.5](n7)(32,67){}
\node[Nframe=y, Nh=1,Nw=1,Nmr=.5, Nfill=y, fillgray=.5 ,linegray=.5](n8)(37,57){}
\node[Nframe=y, Nh=1,Nw=1,Nmr=.5, Nfill=y, fillgray=.5 ,linegray=.5](n9)(42,62){}
\node[Nframe=y, Nh=1,Nw=1,Nmr=.5, Nfill=y, fillgray=.5 ,linegray=.5](n10)(47,57){}
\node[Nframe=y, Nh=1,Nw=1,Nmr=.5, Nfill=y](n11)(52,65){} 
\node[Nframe=y, Nh=1,Nw=1,Nmr=.5, Nfill=y, fillgray=.5 ,linegray=.5](n12)(57,72){}
\node[Nframe=y, Nh=1,Nw=1,Nmr=.5, Nfill=y, fillgray=.5 ,linegray=.5](n13)(62,65){} 
\node[Nframe=y, Nh=1,Nw=1,Nmr=.5, Nfill=y, fillgray=.5 ,linegray=.5](n14)(67,73){}
\node[Nframe=y, Nh=1,Nw=1,Nmr=.5, Nfill=y, fillgray=.5 ,linegray=.5](n15)(72,64){}
\node[Nframe=y, Nh=1,Nw=1,Nmr=.5, Nfill=y, fillgray=.5 ,linegray=.5](n16)(77,71){}
\node[Nframe=y, Nh=1,Nw=1,Nmr=.5, Nfill=y, fillgray=.5 ,linegray=.5](n17)(82,75){}
\node[Nframe=y, Nh=1,Nw=1,Nmr=.5, Nfill=y, fillgray=.5 ,linegray=.5](n18)(87,72){}
\node[Nframe=y, Nh=1,Nw=1,Nmr=.5, Nfill=y](n19)(92,85){}  


\drawedge[ELpos=50, ELside=r, curvedepth=0, AHnb=0, linegray=.5](n1,n2){}
\drawedge[ELpos=50, ELside=r, curvedepth=0, AHnb=0, linegray=.5](n2,n3){}
\drawedge[ELpos=50, ELside=r, curvedepth=0, AHnb=0, linegray=.5](n3,n4){}
\drawedge[ELpos=50, ELside=r, curvedepth=0, AHnb=0, linegray=.5](n4,n5){}
\drawedge[ELpos=50, ELside=r, curvedepth=0, AHnb=0, linegray=.5](n5,n6){}
\drawedge[ELpos=50, ELside=r, curvedepth=0, AHnb=0, linegray=.5](n6,n7){}
\drawedge[ELpos=50, ELside=r, curvedepth=0, AHnb=0, linegray=.5](n7,n8){}
\drawedge[ELpos=50, ELside=r, curvedepth=0, AHnb=0, linegray=.5](n8,n9){}
\drawedge[ELpos=50, ELside=r, curvedepth=0, AHnb=0, linegray=.5](n9,n10){}
\drawedge[ELpos=50, ELside=r, curvedepth=0, AHnb=0, linegray=.5](n10,n11){}
\drawedge[ELpos=50, ELside=r, curvedepth=0, AHnb=0, linegray=.5](n11,n12){}
\drawedge[ELpos=50, ELside=r, curvedepth=0, AHnb=0, linegray=.5](n12,n13){}
\drawedge[ELpos=50, ELside=r, curvedepth=0, AHnb=0, linegray=.5](n13,n14){}
\drawedge[ELpos=50, ELside=r, curvedepth=0, AHnb=0, linegray=.5](n14,n15){}
\drawedge[ELpos=50, ELside=r, curvedepth=0, AHnb=0, linegray=.5](n15,n16){}
\drawedge[ELpos=50, ELside=r, curvedepth=0, AHnb=0, linegray=.5](n16,n17){}
\drawedge[ELpos=50, ELside=r, curvedepth=0, AHnb=0, linegray=.5](n17,n18){}
\drawedge[ELpos=50, ELside=r, curvedepth=0, AHnb=0, linegray=.5](n18,n19){}








\end{picture}
}
    \hrule
      \caption{The result of the computation of $\BestPaths(t_0,v_0,u_0,\psi_0)$. \label{fig:alg-bestpaths}}
  \end{center}
\end{figure*}
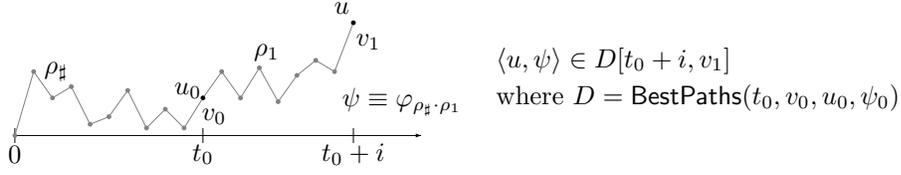

Our algorithm uses the procedure $\BestPaths(t_0,v_0,u_0,\psi_0)$ (shown as Algorithm~\ref{alg:best-path}) that computes 
the $\succeq$-maximal pairs $\tuple{u, \psi}$ corresponding to the paths $\rho_1$ 
of length $1,2, \dots, \abs{V}$ 
that start at time $t_0$ in vertex $v_0$ (see \figurename~\ref{fig:alg-bestpaths}), 
and that prolong a path $\rho_{\sharp}$ with sum of weight $u_0$ and constraint 
$\psi_0$ on $M$ (where $u$ is the sum of weights along $\rho_{\sharp} \cdot \rho_1$, 
and $\psi \equiv \varphi_{\rho_{\sharp} \cdot \rho_1}$). 
We give a precise statement of this result in Lemma~\ref{lem:bestpaths}.

\begin{lemma}[Correctness of $\BestPaths$]\label{lem:bestpaths}
Let $\rho_{\sharp}$ be a finite path of length $t_0$, that ends in state $\ends(\rho_{\sharp}) = v_0$
with sum of weight $u_0$ and associated constraint $\psi_0$ on $M$.
Let $D = \BestPaths(t_0,v_0,u_0,\psi_0)$. Then, 
\begin{itemize}
\item for all $0 \leq i \leq \abs{V}$,
for all $v_1 \in V$, for all pairs $\tuple{u, \psi} \in D[t_0+i,v_1]$, there
exists a path $\rho_1$ of length $i$ with $\starts(\rho_1) = v_0$ and $\ends(\rho_1) = v_1$,
such that 
\begin{itemize}
\item $u$ is the sum of weights of the path $\rho_{\sharp} \cdot \rho_1$, and 
\item $\psi \equiv \varphi_{\rho_{\sharp} \cdot \rho_1}$ is the constraint on $M$
associated with the path $\rho_{\sharp} \cdot \rho_1$;
\end{itemize}

\item
for all paths $\rho_1$ of length $i \leq \abs{V}$ such that $\starts(\rho_1) = v_0$
and $\ends(\rho_1) = v_1$,
there exists a pair $\tuple{u', \psi'} \in D[t_0+i,v_1]$ such that 
$\tuple{u', \psi'} \succeq \tuple{u, \psi}$ where 
\begin{itemize}
\item $u$ is the sum of weights of the path $\rho_{\sharp} \cdot \rho_1$, and 
\item $\psi \equiv \varphi_{\rho_{\sharp} \cdot \rho_1}$ is the constraint on $M$
associated with the path $\rho_{\sharp} \cdot \rho_1$.
\end{itemize}
\end{itemize}
\end{lemma}

\begin{proof}
For the first item, the proof is by induction on $i$. The case $i=0$ holds since
$D[t_0,v_1]$ is nonempty only for $v_1 = v_0$ (lines~\ref{alg:init-begin}-\ref{alg:init-end}
of Algorithm~\ref{alg:best-path}), and we can take for $\rho_1$ the empty
path since then $D[t_0,v_0] = \{ \tuple{u_0, \psi_0} \}$ contains the pair associated
with $\rho_{\sharp} = \rho_{\sharp} \cdot \rho_1$. 

For the inductive case, consider length $i \geq 1$ and assume that the result holds for length $i-1$.
Then for all pairs $\tuple{u_1, \psi_1} \in D[t_0+i-1,v_1]$ where $v_1 \in V$ 
(see also line~\ref{alg:best-path-for-D} of Algorithm~\ref{alg:best-path}),
there exists a path $\rho_1$ of length $i-1$ such that $\tuple{u_1, \psi_1}$ is the pair associated 
with $\rho_{\sharp} \cdot \rho_1$. It is easy to see that the pair $\tuple{u, \psi}$ 
added to $D[t_0 + i,v]$ at line~\ref{alg:best-path-add-D} of Algorithm~\ref{alg:best-path}
is associated with the path 
$\rho_{\sharp} \cdot \rho_1 \cdot (v_1,v)$ where
$u = u_1 + w(v_1,v)$ and $\psi \equiv \psi_1 \land (r \geq M \cdot (t-T))$
with $t = t_0 + i - 1 = \abs{\rho_{\sharp} \cdot \rho_1 \cdot (v_1,v)} - 1$. 
Since the assignment at line~\ref{alg:best-path-max-D} of Algorithm~\ref{alg:best-path}
can only remove pairs from $D[t_0 + i,v]$, the result follows.

For the second item, the result follows from similar arguments as above, 
a proof by induction on $i$ using Lemma~\ref{lem:constraint-prefered}, and 
the fact that the algorithm explores all successors $v$ of each vertex $v_1$ 
that ends a path associated with a pair $\tuple{u_1, \psi_1} \in D[t_0+i-1,v_1]$.
\end{proof}

As we know that simple lassos are sufficient for optimal value (Lemma~\ref{lem:simple-lasso-suffice}), 
our algorithmic solution is to explore finite paths from the initial vertex,
until a loop is formed. Thus it is sufficient to explore paths of length 
at most $\abs{V}$. However, given a simple lasso $\rho_{\!A}^{} \cdot \rho_C^{\omega}$, 
it is not sufficient that the finite path $\rho_{\!A}^{} \cdot  \rho_{C}^{}$ 
lies above a line $M \cdot (t-T)$ (where $M$ satisfies the constraint $\psi_{\!AC}^{}$ associated with 
$\rho_{\!A}^{} \cdot  \rho_{C}^{}$)
to ensure that the value of the lasso $\rho_{\!A}^{} \cdot \rho_C^{\omega}$
is at least~$0$. The reason is that by repeating the cycle $\rho_{C}^{}$
several times, the path may eventually cross the line $M \cdot (t-T)$. 
We show (in Lemma~\ref{lem:constraint-cycle}) that this cannot happen if the average weight $M_C$ of the cycle 
is greater than the slope of the line (i.e., $M_C \geq M$).

\begin{algorithm}[t]
\caption{$\ExistsPositivePath(v_0)$}
\begin{flushleft}
\label{alg:exists-positive-path}
 \AlgData{$v_0$ is an initial vertex.}
 \AlgResult{$\true$ iff there exists a path from $v_0$ with expected utility at least $0$.}
 \Begin{
	\nl $A \gets \BestPaths(0,v_0,0,\true)$ \label{alg:exists-positive-path-first-call}\;
        \nl \For{$i = 0, \dots, \abs{V}$\negskip \label{alg:exists-positive-path-for-i}} 
        {
          \nl \For{$\hat{v} \in V$ and $\tuple{u_1,\psi_1} \in A[i,\hat{v}]$\negskip \label{alg:exists-positive-path-for-pair1}}    
             {
                \nl $C \gets \BestPaths(i,\hat{v},u_1,\psi_1)$ \label{alg:exists-positive-path-second-call}  \;
                \nl \For{$j = 1, \dots, \abs{V}-i$\negskip \label{alg:exists-positive-path-for-j}} 
                {
                   \nl \For{$\tuple{u_2,\psi_2} \in C[i+j,\hat{v}]$\negskip \label{alg:exists-positive-path-for-pair2}}
                   {
                      \nl \lIf{$\psi_2 \land \frac{u_2 - u_1}{j} \geq M \text{ is satisfiable} $}{\KwRet{$\true$}} \label{alg:exists-positive-path-return-true} \;
                   }
                }
             }
        }
        \nl \KwRet{$\false$}\; 
 }
\end{flushleft}
\end{algorithm}

\begin{lemma}\label{lem:constraint-cycle}
Given a lasso $\rho_{\!A}^{} \cdot \rho_C^{\omega}$, let $\psi_{\!AC}^{}$ be the symbolic constraint
on $M$ associated with the finite path $\rho_{\!A}^{} \cdot  \rho_{C}^{}$, and let $M_C$
be the average weight of the cycle $\rho_{C}^{}$.
The lasso $\rho_{\!A}^{} \cdot \rho_C^{\omega}$ has value at least $0$ if and only if
the formula $\psi_{\!AC}^{} \land (M_C \geq M)$ is satisfiable.
\end{lemma} 

\begin{proof}
First, if the lasso $\rho_{\!A}^{} \cdot \rho_C^{\omega}$  has value at least $0$,
then by Lemma~\ref{lem:exists-slope}, there exists a slope $M \in \real$ 
such that $u_t \geq M \cdot (t-T)$ for all $t \geq 0$ (where $u_t$ is the sum
of weights at time $t$ in $\rho_{\!A}^{} \cdot \rho_C^{\omega}$). For such value of $M$,
the formula $\psi_{\!AC}^{}$ holds (by definition), and it is easy to see that $M_C \geq M$
(otherwise, there would exist $t \geq 0$ such that $u_t < M \cdot (t-T)$).
Therefore $\psi_{\!AC}^{} \land (M_C \geq M)$ is satisfiable.

Second, if the formula $\psi_{\!AC}^{} \land (M_C \geq M)$ is satisfiable, then let $M$ be a
satisfying value, and by Lemma~\ref{lem:exists-slope} and a similar argument
as above, the lasso $\rho_{\!A}^{} \cdot \rho_C^{\omega}$ has value at least $0$.
\end{proof}

The algorithm $\ExistsPositivePath(v_0)$ 
explores the paths from $v_0$, and keeps the $\succeq$-preferred paths, that is those
with the largest total weight and weakest constraint on $M$. 
There may be several $\succeq$-incomparable paths of a given length $i$
that reach a given vertex $\hat{v}$, therefore we need to compute 
a \emph{set} $A[i,\hat{v}]$ of $\succeq$-incomparable pairs (line~\ref{alg:exists-positive-path-first-call} 
of Algorithm~\ref{alg:exists-positive-path}). 

Given a pair $\tuple{u_1,\psi_1} \in A[i,\hat{v}]$, the algorithm $\ExistsPositivePath$
further explores (for-loop at line~\ref{alg:exists-positive-path-for-pair1} of Algorithm~\ref{alg:exists-positive-path})
the paths from $\hat{v}$, until a cycle $\rho_{C}^{}$ of length $j$ is formed around $\hat{v}$,
with average weight $M_C = \frac{u_2 - u_1}{j}$ and associated pair $\tuple{u_2,\psi_2} \in C[i+j,\hat{v}]$ 
(line~\ref{alg:exists-positive-path-return-true} of Algorithm~\ref{alg:exists-positive-path})
such that $\psi_2 \land (M_C \geq M)$ is satisfiable. 
We claim that there exists such a cycle if and only if there exists a lasso with value at least $0$.
The claim is established in the following lemma.

\begin{lemma}[Correctness of $\ExistsPositivePath$]\label{lem:ExistsPositivePath}
There exists an infinite path from $v_0$ with value at least $0$
if and only if $\ExistsPositivePath(v_0)$ returns $\true$.
\end{lemma}

\begin{longversion}
\begin{proof}
{\bf (First part)} 

For the first direction of the proof, if there exists an infinite path with value at least~$0$, 
then by Lemma~\ref{lem:simple-lasso-suffice}
there exists a lasso $\rho = \rho_{\!A}^{} \cdot \rho_C^{\omega}$ with value at least~$0$. 

Consider the call $A \gets \BestPaths(t_0,v_0,u_0,\psi_0)$ in $\ExistsPositivePath$ 
(line~\ref{alg:exists-positive-path-first-call} of Algorithm~\ref{alg:exists-positive-path})
where $t_0 = u_0 = 0$ and $\psi_0 \equiv \true$.
Let $\hat{v} = \ends(\rho_{\!A}^{})$ and let $i$ be the length of $\rho_{\!A}^{}$ (note that $i < \abs{V}$
because $\rho_{\!A}^{}$ is acyclic).
By the correctness result of $\BestPaths$ (Lemma~\ref{lem:bestpaths} (item~$2$), where $\rho_{\sharp}$ is the empty path),
there is a pair $\tuple{u_1,\psi_1} \in A[i,\hat{v}]$ such that 
$\tuple{u_1,\psi_1} \succeq \tuple{u_A,\psi_A}$ where $\tuple{u_A,\psi_A}$
is the pair associated with $\rho_{\!A}^{}$, thus $u_1 \geq u_A$ and $\psi_A \to \psi_1$
holds. Then by Lemma~\ref{lem:bestpaths} (item~$1$),
there is a path $\rho_1$ of length $i$ from $v_0$ to $\hat{v}$, and
$u_1$ is the sum of weights of $\rho_1$, 
and $\psi_1 \equiv \varphi_{\rho_1}$ is the constraint on $M$
associated with $\rho_1$ (i.e., $\rho_1$ is preferred to $\rho_{\!A}^{}$).

Now consider the call $C \gets \BestPaths(i, \hat{v}, u_1,\psi_1)$ in $\ExistsPositivePath$ 
(line~\ref{alg:exists-positive-path-second-call} of Algorithm~\ref{alg:exists-positive-path}).
Let $\rho_{\sharp} = \rho_1$ in Lemma~\ref{lem:bestpaths} and note that 
the assumptions of that lemma are satisfied, namely $\tuple{u_1,\psi_1}$ is the pair
associated with $\rho_1$, and $\hat{v} = \ends(\rho_1)$.

Since $\rho_{\!A}^{} \cdot \rho_C^{\omega}$ is a lasso, we have 
$\starts(\rho_{C}^{}) = \ends(\rho_{C}^{}) = \ends(\rho_{\!A}^{}) = \hat{v}$ 
and let $j$ be the length of $\rho_{C}^{}$ (note that $i+j \leq \abs{V}$).
By Lemma~\ref{lem:bestpaths} (item~$2$),
there is a pair $\tuple{u_2,\psi_2} \in C[i+j,\hat{v}]$ such that 
$\tuple{u_2,\psi_2} \succeq \tuple{u_{1C},\psi_{\!1C}^{}}$ where $\tuple{u_{1C},\psi_{\!1C}^{}}$
is the pair associated with $\rho_1 \cdot \rho_{C}^{}$, thus $u_2 \geq u_{1C}$ and $\psi_{\!1C}^{} \to \psi_2$
holds, and by Lemma~\ref{lem:bestpaths} (item~$1$), there is a path $\rho_2$ of length $j$ 
such that $\starts(\rho_2) = \ends(\rho_2) = \hat{v}$ and
$u_2$ is the sum of weights of $\rho_1 \cdot \rho_2$, 
and $\psi_2 \equiv \varphi_{\rho_1 \cdot \rho_2}$ is the constraint on $M$
associated with $\rho_1 \cdot \rho_2$.

Now we show that $\psi_2 \land \frac{u_2 - u_1}{j} \geq M$ is satisfiable, and
thus $\ExistsPositivePath(v_0)$ returns $\true$ (Line~\ref{alg:exists-positive-path-return-true} of Algorithm~\ref{alg:exists-positive-path}).
First, by Lemma~\ref{lem:constraint-cycle} the formula $\psi_{\!AC}^{} \land (M_C \geq M)$
is satisfiable, and by Lemma~\ref{lem:constraint-prefered} we have $\psi_{\!AC}^{} \to \psi_{\!1C}^{}$.
We showed above that 
$\psi_{\!1C}^{} \to \psi_2$, thus $\psi_2 \land (M_C \geq M)$ is satisfiable.
Now, since the length of the cycle $\rho_{C}^{}$ (and of $\rho_2$) is $j-i$ (i.e., the length of $\rho_{\!A}^{} \cdot \rho_{C}^{}$
minus the length of $\rho_{\!A}^{}$), we have  $M_C = \frac{S_C}{j}$. Moreover
we showed above that $u_2 \geq u_{1C} = u_1 + S_C$, thus $M_C = \frac{S_C}{j}\leq \frac{u_2 - u_1}{j}$,
and since $\psi_2 \land (M_C \geq M)$ is satisfiable
it follows that $\psi_2 \land \frac{u_2 - u_1}{j} \geq M$ is satisfiable as well.

{\bf (Second part)} 

For the second direction of the proof, if $\ExistsPositivePath(v_0)$ returns $\true${},
then there exists $i,j, \hat{v}, \tuple{u_1,\psi_1}, \tuple{u_2,\psi_2}$ 
(corresponding to the for-loops in 
lines~\ref{alg:exists-positive-path-for-i},~\ref{alg:exists-positive-path-for-pair1},~\ref{alg:exists-positive-path-for-j},~\ref{alg:exists-positive-path-for-pair2} of Algorithm~\ref{alg:exists-positive-path}) such that:   

\begin{itemize}
\item $0 \leq i \leq \abs{V}$ and $1 \leq j \leq \abs{V}-i$,
\item $\hat{v} \in V$,
\item $\tuple{u_1,\psi_1} \in A[i,\hat{v}]$ and $\tuple{u_2,\psi_2} \in C[i+j,\hat{v}]$ 
 where $A = \BestPaths(0,v_0,0,\true)$, and  $C = \BestPaths(i,\hat{v},u_1,\psi_1)$, 
\item $\psi_2 \land \frac{u_2 - u_1}{j} \geq M$ is satisfiable.
\end{itemize}
\noindent Therefore, by Lemma~\ref{lem:bestpaths} (item~$1$), there exist paths $\rho_{\!A}^{}$ and $\rho_{C}^{}$ such that:
\begin{itemize}
\item $\rho_{\!A}^{}$ is a path of length $i$ from $v_0$ to $\hat{v}$,
such that $u_1$ is the sum of weights of the path $\rho_{\!A}^{}$, 
and $\psi_1 \equiv \varphi_{\rho_{\!A}^{}}$;

\item $\rho_{C}^{}$ is a path of length $j$ with $\starts(\rho_{C}^{}) = \ends(\rho_{C}^{}) = \hat{v}$
(thus $\rho_{C}^{}$ is a cycle), such that 
$u_2$ is the sum of weights of the path $\rho_{\!A}^{} \cdot \rho_{C}^{}$, 
and $\psi_2 \equiv \varphi_{\rho_{\!A}^{} \cdot \rho_{C}^{}}$ is the constraint on $M$
associated with the path $\rho_{\!A}^{} \cdot \rho_{C}^{}$.
\end{itemize}

Therefore, $u_2 - u_1$ is the sum of the weights along $\rho_{C}^{}$, and thus $M_C = \frac{u_2 - u_1}{j}$.
Since the formula $\psi_2 \land \frac{u_2 - u_1}{j} \geq M$ is satisfiable,
it follows that $\varphi_{\rho_{\!A}^{} \cdot \rho_{C}^{}} \land (M_C \geq M)$ is satisfiable,
and by Lemma~\ref{lem:constraint-cycle}, the lasso $\rho_{\!A}^{} \cdot \rho_C^{\omega}$ has value at least $0$.  
\end{proof}
\end{longversion}

\begin{longversion}
\medskip
\noindent{\em Optimal value.}
We can compute the optimal value using the procedure $\ExistsPositivePath$ as follows.
From Lemma~\ref{lem:geometric-interpretation}, the optimal value is either of the form 
$\frac{u_{t_1} \cdot (t_2 - T) + u_{t_2} \cdot (T-t_1)}{t_2 - t_1}$, or
of the form $u_{t_1} + (T - t_1) \cdot \nu$ where the following bounds hold
($\nu = \inf_{t_2 \geq T} \,\, \frac{u_{t_2} - u_{t_1}}{t_2 - t_1}$):
\begin{itemize}
\item $0 \leq t_1 \leq t_2 \leq \abs{V}$
\item $0 \leq t_2 - t_1 \leq \abs{V}$
\item $0 \leq T - t_1 \leq \abs{V}$
\item $0 \leq t_2 - T \leq \abs{V}$
\item $-W \cdot \abs{V} \leq u_{t_1}, u_{t_2} \leq W \cdot \abs{V}$
\item $\nu$ is a rational number $\frac{p}{q}$ where $-W \cdot \abs{V} \leq p \leq W \cdot \abs{V}$ and $1 \leq q \leq \abs{V}$
\end{itemize}

\noindent Therefore, in both cases we get the following result.

\begin{lemma}\label{lem:optimal-value}
The optimal value belongs to the set
$${\sf ValueSpace}  = \left\{\frac{p}{q} \mid -2W \cdot \abs{V}^2 \leq p \leq 2W \cdot \abs{V}^2 \text{ and } 1 \leq q \leq \abs{V} \right\}.$$
\end{lemma}

Given a value $\frac{p}{q}$, we can decide if there exists a path with expected value
at least $\frac{p}{q}$ by subtracting $\frac{p}{q \cdot T}$ from all the weights the graphs,
and asking if there exists a path with expected value at least $0$ in the modified graph.
Indeed, if we define $w'(e) = w(e) + \eta$ for all edges $e \in E$, 
then for all paths~$\rho$, if $u$ is the sequence of utilities along $\rho$ according to $w$, 
and $u'$ is the sequence of utilities along $\rho$ according to $w'$, then
$$\sum_i p_i \cdot u'_{i}  = \sum_i p_i \cdot (u_{i} + \eta \cdot i) = \eta \cdot \sum_i p_i \cdot i + \sum_i p_i \cdot u_{i} = T\cdot \eta + \sum_i p_i \cdot u_{i},$$

thus the value of the path is shifted by $T\cdot \eta$.
Then it follows from Lemma~\ref{lem:optimal-value} that the optimal 
value can be computed by a binary search using $O(\abs{{\sf ValueSpace}}) = O(\log(W \cdot \abs{V}))$ calls
to $\ExistsPositivePath$.

\medskip
\noindent{\em Optimal path.}
An optimal path can be constructed by a slight modification of the algorithm.
In $\BestPaths$, we can maintain a path associated to each pair in $D$
as follows: the empty path is associated to the pair $\tuple{u_0,\psi_0}$
added at line~\ref{alg:init-begin} of Algorithm~\ref{alg:best-path}, and
given the path $\rho_1$ associated with the pair $\tuple{u_1,\psi_1}$ 
(line~\ref{alg:best-path-for-D} of Algorithm~\ref{alg:best-path}),
we associate the path $\rho_1 \cdot (v_1,v)$ with the pair $\tuple{u,\psi}$
added to $D$ at line~\ref{alg:best-path-add-D} of Algorithm~\ref{alg:best-path}.
It is easy to see that for every pair $\tuple{u,\psi}$ in $D$, 
the associated path can be used as the path $\rho_1$ in Lemma~\ref{lem:bestpaths} (item~$1$).
Therefore, when $\ExistsPositivePath(v_0)$ returns $\true$ (line~\ref{alg:exists-positive-path-return-true}
of Algorithm~\ref{alg:exists-positive-path}), we can output the path $\rho_1 \cdot \rho_2^\omega$
where $\rho_i$ is the path associated with the pair $\tuple{u_i,\psi_i}$ ($i= 1,2$).

\end{longversion}

\begin{shortversion}
\medskip
\noindent{\em Optimal value and plan.}
The optimal value can be computed by a binary search using $O(\log(W \cdot \abs{V}^3))$ calls
to $\ExistsPositivePath$.
An optimal path can be constructed by a slight modification of the algorithm.
In $\BestPaths$, we can maintain a path associated to each pair in $D$
as follows: the empty path is associated to the pair $\tuple{u_0,\varphi_0}$
added at line~\ref{alg:init-begin} of Algorithm~\ref{alg:best-path}, and
given the path $\rho_1$ associated with the pair $\tuple{u_1,\varphi_1}$ 
(line~\ref{alg:best-path-for-D} of Algorithm~\ref{alg:best-path}),
we associate the path $\rho_1 \cdot (v_1,v)$ with the pair $\tuple{u,\varphi}$
added to $D$ at line~\ref{alg:best-path-add-D} of Algorithm~\ref{alg:best-path}.
It is easy to see that for every pair $\tuple{u,\varphi}$ in $D$, 
the associated path can be used as the path $\rho_1$ in Lemma~\ref{lem:bestpaths} (item~$1$).
Therefore, when $\ExistsPositivePath(v_0)$ returns $\true$ (line~\ref{alg:exists-positive-path-return-true}
of Algorithm~\ref{alg:exists-positive-path}), we can output the path $\rho_1 \cdot \rho_2^\omega$
where $\rho_i$ is the path associated with the pair $\tuple{u_i,\varphi_i}$ ($i= 1,2$).
\end{shortversion}

\subsubsection{Complexity analysis}
We present the running-time analysis of $\ExistsPositivePath$ (Algorithm~\ref{alg:exists-positive-path}).
The key challenge is to bound the number of $\succeq$-incomparable pairs. 
The number of such pairs corresponds to the number of simple paths in a graph, and hence can be
exponential in general. 
Our main argument is to establish a polynomial bound on the number of $\succeq$-incomparable pairs.

To analyze the complexity of the algorithm, we need to bound the
size of the array $D$ computed by $\BestPaths$ (Algorithm~\ref{alg:best-path}).
We show that there cannot be too many different pairs in a given 
entry $D[t_0+i,v_1]$. By Lemma~\ref{lem:bestpaths}, to each pair $\tuple{u,\psi} \in D[t_0+i,v_1]$
we can associate a path $\rho$ of length $i$ with $\starts(\rho) = v_0$ and $\ends(\rho) = v_1$,
such that (our analysis holds for all paths $\rho_{\sharp}$ in Lemma~\ref{lem:bestpaths},
and as $\rho_{\sharp}$ plays no role in the argument, we proceed with empty $\rho_{\sharp}$
for simplicity of the exposition\footnote{The proof can be carried out 
analogously by considering $\rho_{\sharp} \cdot \rho$ instead of $\rho$ with 
heavier notation.}):
\begin{itemize}
\item $u$ is the sum of weights of the path $\rho$, and 
\item $\psi \equiv \varphi_{\rho}$ is the constraint on $M$
associated with the path $\rho$.
\end{itemize}

It is important to note that the constraint $\psi$ is determined by (at most) 
two points $t_L, t_R$ in $\rho$ (see also \figurename~\ref{fig:preferred} and \figurename~\ref{fig:preferred2}),
one before~$T$ and one after~$T$,
namely 
$$ \psi \equiv \big(u_{t_L} \geq M \cdot (t_L-T) \big) \land \big(u_{t_R} \geq M \cdot (t_R-T) \big) $$
where $t_L = \argmax_{0 \leq i \leq T}(\frac{u_i}{i-T})$
and $t_R = \argmin_{T \leq i \leq \abs{\rho}}(\frac{u_i}{i-T})$.

Note that the first constraint in the above expression is a lower bound on $M$
since $t_L \leq T$,
and the second constraint (which may not exist, if $\abs{\rho} < T$)
is an upper bound on $M$. For simplicity of exposition, we assume that $\abs{\rho} \geq T$.
The case $\abs{\rho} < T$ is handled analogously ($t_R$ is undefined in that case).

Define the \emph{down-point} of $\rho = e_0 e_1 \dots e_{\abs{\rho}-1}$ as $\downpoint(\rho) = \tuple{t_L,v_L,t_R,v_R}$
where $t_L$ and $t_R$ are defined above, and $v_L = \ends(e_0 e_1 \dots e_{t_L})$,
and $v_R = \ends(e_0 e_1 \dots e_{t_R})$ (for $\abs{\rho} < T$, the down-point
of $\rho$ is $\downpoint(\rho) = \tuple{t_L,v_L}$).

Decompose $\rho$ into $\rho_L = e_0 e_1 \dots e_{t_L}$, $\rho_M = e_{t_L + 1} e_{t_L + 2} \dots e_{t_R}$,
and $\rho_R = e_{t_R + 1} e_{t_R + 2} \dots e_{\abs{\rho}-1}$.
We claim that the paths corresponding to two different pairs in $D[t_0+i,v_1]$
have different down-points, which will give us a polynomial bound on the size of $D[t_0+i,v_1]$.
Intuitively, and towards contradiction, if two down-points are the same in two
different paths, then we can select the best pieces among $(\rho_L,\rho_M,\rho_R)$
from the two paths and construct a path that is preferred,
and thus whose pair is in $D[t_0+i,v_1]$ and subsumes some pair in $D[t_0+i,v_1]$,
which is a contradiction since the elements of $D[t_0+i,v_1]$
are $\succeq$-maximal.

\begin{lemma}\label{lem:downpoints}
Let $D = \BestPaths(t_0,v_0,u_0,\psi_0)$ and $1 \leq i\leq \abs{V}$.
For all pairs $\tuple{u,\psi}, \tuple{u',\psi'} \in D[t_0+i,v_1]$, let $\rho, \rho'$ be their
respective associated path; if $\tuple{u,\psi} \neq \tuple{u',\psi'}$,
then the down-points of $\rho$ and $\rho'$ are different 
($\downpoint(\rho) \neq \downpoint(\rho')$).
\end{lemma}

\begin{longversion}
\begin{proof}
We prove the contrapositive, for $\abs{\rho} \geq T$ (the case $\abs{\rho} < T$ is simpler, and proved analogously). 
Assume that $\tuple{t_L,v_L,t_R,v_R} = \tuple{t'_L,v'_L,t'_R,v'_R}$ (the down-points are equal),
and we show that then $\tuple{u,\psi} = \tuple{u',\psi'}$. 

First, since $t_L = t'_L$ and $v_L = v'_L$, we claim that the sum of weights 
at time $t_L$ is the same in $\rho$ and in $\rho'$, that is $u_{t_L} = u'_{t_L}$,
and therefore, $\varphi_{\rho_{L}} \equiv \varphi_{\rho'_{L}}$ (remember that the 
constraint $\psi$ associated with $\rho$ and $\rho'$ is determined by $t_L = t'_L$). 
The proof of this claim is by contradiction. Assume that $u_{t_L} > u'_{t_L}$ (the argument
for the case $u_{t_L} < u'_{t_L}$ is analogous). Consider the path $\overline{\rho} = \rho_{L} \cdot \rho'_{M} \cdot \rho'_{R}$,
and note that $\overline{\rho}$ is indeed a path\footnote{Note that if $\rho$ 
and $\rho'$ have a common prefix (such as $\rho_{\sharp}$), then $\overline{\rho}$ also
has the same prefix.}, 
as $\ends(\rho_{L}) = v_L = v'_L = \starts(\rho'_{M})$.
Comparing $\overline{\rho}$ and $\rho'$, since $u_{t_L} > u'_{t_L}$ it is easy to see that $\bar{u} > u'$ where 
$\bar{u}$ is the sum of weights of $\overline{\rho}$, and by the same argument
we have $\psi' \to \psi_{\overline{\rho}}$. It follows that $\overline{\rho}$
is preferred to $\rho'$, and by Lemma~\ref{lem:bestpaths} the set $D[t_0+i,v_1]$ 
contains a pair $\tuple{u^* , \psi^*} \succeq \tuple{\bar{u}, \varphi_{\overline{\rho}}} \succeq \tuple{u', \psi'}$.
Since $D[t_0+i,v_1]$ is a set of $\succeq$-maximal elements (line~\ref{alg:best-path-max-D} 
of Algorithm~\ref{alg:best-path}), it follows that $\tuple{u', \psi'} \not\in D[t_0+i,v_1]$,
in contradiction with the assumption of the lemma.

Second, by an analogous argument, since $t_R = t'_R$ and $v_R = v'_R$, the sum of weights 
at time $t_R$ is the same in $\rho$ and in $\rho'$, that is $u_{t_R} = u'_{t_R}$,
and therefore, $\varphi_{\rho_{R}} \equiv \varphi_{\rho'_{R}}$. Finally 
$u = u'$ and $\psi \equiv \psi'$, which concludes the proof.
\end{proof}
\end{longversion}

It follows from Lemma~\ref{lem:downpoints} that the size of all sets $D[t_0 + i,v_1]$
for $1 \leq i\leq \abs{V}$ and $v_1 \in V$ is at most $\abs{V}^4$, the maximum number of
different down-points.

We now show that the worst-case complexity of $\BestPaths$ and $\ExistsPositivePath$
is polynomial, and thus the optimal expected value problem is solvable in polynomial
time. 

The worst-case complexity of $\BestPaths$ is $O(\abs{V}^{10})$, as there are
two nested for-loops over $V$ (line~\ref{alg:best-path-first-for} and line~\ref{alg:best-path-second-for}
in Algorithm~\ref{alg:best-path}), in which the dominating
operation is the computation of the $\succeq$-maximal elements of $D[t_0 + i,v]$
(line~\ref{alg:best-path-max-D}), which is quadratic in the size of $D[t_0 + i,v]$,
thus in $O(\abs{V}^8)$.

The worst-case complexity of $\ExistsPositivePath$ is $O(\abs{V} \cdot \abs{V} \cdot \abs{V}^4 \cdot \abs{V}^{10}) = O(\abs{V}^{16})$,
as a product of the size of the three outermost for-loops, and
the dominating call to $\BestPaths$ (line~\ref{alg:exists-positive-path-second-call}) in $O(\abs{V}^{10})$.
Therefore we obtain Theorem~\ref{thm:algo}.

\section{Conclusion}
In this work we consider the expected finite-horizon problem. 
Our most interesting results are for the case of adversarial distribution of stopping 
times, for which we establish stationary plans are sufficient, and present polynomial-time 
algorithms.
In terms of algorithmic complexity, our main goal was to establish polynomial-time algorithms,
and we expect that better algorithms and refined complexity analysis can be obtained.

\bibliographystyle{plain}
\bibliography{biblio}

\end{document}